\documentclass{article}

\PassOptionsToPackage{round}{natbib}
\usepackage[main,final]{neurips_2026}

\usepackage[utf8]{inputenc}
\usepackage[T1]{fontenc}
\usepackage{amsmath}
\usepackage{amssymb}
\usepackage{mathtools}
\usepackage{graphicx}
\usepackage{tikz}
\usepackage{float}
\usepackage{algorithm}
\usepackage{algpseudocode}
\usepackage{booktabs}
\usepackage{hyperref}
\usepackage{url}
\usepackage{xcolor}
\usepackage{microtype}
\usetikzlibrary{shapes, positioning, calc, arrows.meta, fit, decorations.pathreplacing}

\usepackage{amsthm}
\newtheorem{definition}{Definition}[section]
\newtheorem{theorem}{Theorem}[section]

\newtheorem{lemma}{Lemma}[section]
\newtheorem{corollary}{Corollary}[section]
\newtheorem{assumption}{Assumption}[section]
\theoremstyle{remark}
\newtheorem{remark}{Remark}[section]

\title{Bandits via Additive Quantized Representations}

\author{%
  Ami Tavory \quad Noam Touitou \quad Tal Sarig \quad
  Frank Cheng \quad Ido Guy \\
  Meta Platforms
}

\begin{document}

\maketitle

\begin{abstract}

Contextual bandits require balancing nonlinear reward modeling with
online efficiency. Tree ensembles and neural methods capture
nonlinearities but require periodic retraining and large replay
buffers. Linear models update efficiently per observation with
$O(1)$ memory, but are fundamentally restricted to linear reward
structures. We propose
Residual Quantization (RQ) as a representation layer to bridge this
gap. An offline-trained RQ codebook maps continuous contexts into
discrete centroid assignments across $\ell$ levels, set dynamically
through a shadow mechanism. This enables a spectrum of additive bandit
algorithms that achieve nonlinear expressivity with strictly bounded
memory. Across 13 datasets, RQ variants beat their
non-RQ counterparts on 11 of 13 datasets, often by wide
margins, while matching doubling-retrain XGBoost and neural
baselines at up to $1000\times$ less memory.

\end{abstract}

\section{Introduction}\label{sec:intro}

In contextual bandits, an agent repeatedly observes a context,
selects an action, and receives a reward, learning to improve its
choices from each interaction. Every decision is made once and in
real time, so the algorithm should ideally update from each
observation as it arrives.

Tree ensembles~\cite{nilsson2024tets} and neural
methods~\cite{zhou2020neural,zhang2021neural} capture nonlinear
reward structure, but they periodically retrain on a
replay buffer that can grow linearly with the horizon.
At five million rounds on a 28-feature dataset, a doubling-retrain
bandit (retraining at exponentially spaced intervals on all past
data) consumes over 1\,GB of replay memory
(Figure~\ref{fig:pareto}).
Linear methods (LinTS~\cite{agrawal2013thompson},
LinUCB~\cite{abbasi2011improved}) avoid memory and spiky retraining
issues, but are structurally restricted to linear rewards.

This is not a smooth tradeoff. Practitioners face a binary choice:
expressive models with unbounded memory, or bounded memory with
linear rewards. The gap between the two regimes, nonlinear
expressiveness at fixed memory cost, is the problem we address.

\smallskip
\noindent\textbf{Our approach.}
We use Residual Quantization (RQ)~\cite{chen2010approximate} as a
bandit representation layer to enable gradient-boosted reward
modeling with fixed $O(1)$ memory.
We train an RQ codebook of $\ell$ levels with $b$ centroids each
via unsupervised $k$-means on unlabeled contexts (assumed
available and cheap) and freeze it before the bandit starts.
RQ codebooks were originally developed for efficient retrieval;
we repurpose them for gradient-boosted bandits, where we fit
a bandit model at each level to predict the \emph{pseudo-target}, the
gap between the observed reward and the cumulative
prediction of shallower levels. With
$\ell$ levels and $b$ centroids, this yields $b^\ell$ effective
cells using only $\ell b$ parameters, a potentially exponential
compression.
The online phase requires no replay buffer, no retraining, and
fixed $O(1)$ memory at every horizon.

The algorithm adapts depth online via a \emph{shadow}
mechanism: a shadow model trains at the next depth but does not
influence arm selection, so the active model's performance is
unaffected during evaluation. A statistical test at pre-committed
check times promotes the shadow if it improves prediction;
otherwise the method defaults to its initial depth
(Section~\ref{sec:drq}).

\smallskip

We evaluate the architecture across three variants that span
the efficiency--robustness spectrum:
\emph{TS-RQ} ($O(\ell K)$/round) uses per-centroid scalar
counters with no features, robust and simple.
\emph{SGD-LinTS-RQ} ($O(\ell d K)$) fits diagonal linear models on
the context residuals at each level.
\emph{LinTS-RQ} ($O(\ell d^2 K)$) fits full-covariance LinTS, the
quality ceiling, competitive with XGBoost and NeuralCB at $1000\times$ less
memory when data per parameter suffices
(Section~\ref{sec:exp-xgb}).
All three use shadow promotion and maintain fixed memory.

\smallskip
\noindent\textbf{Contributions.}
\begin{enumerate}
\item \emph{RQ as a bandit model}: we repurpose
      residual quantization from retrieval to bandits, using the
      additive codebook as a nonlinear reward model that maintains
      fixed $O(1)$ memory and requires no replay buffer.
\item \emph{Analysis}: under three base-learner convergence
      properties (verified for all variants), the asymptotic
      regret rate is bounded by the additive misspecification
      $\Delta_\ell$ (Theorem~\ref{thm:asymptotic}).  We decompose
      depth-dependent approximation error exactly into saturated
      quantization and additive-interaction terms, and prove
      geometric convergence to the irreducible floor when
      reconstruction and additive-tail errors contract
      geometrically (Theorem~\ref{thm:depth-rate} and
      Corollary~\ref{cor:depth-rate}).
      A shadow mechanism adapts depth online with
      false-promotion probability $\le \alpha$
      (Theorem~\ref{thm:shadow}).
\item \emph{Scale experiments}: we evaluate on 13 datasets
      with up to $11$M rounds and 30 seeds each, showing that
      RQ methods improve over their within-class baselines on
      11 of 13 datasets and that LinTS-RQ is competitive with XGBoost (7W/4L/2T) at
      $1000\times$ less memory.
\end{enumerate}

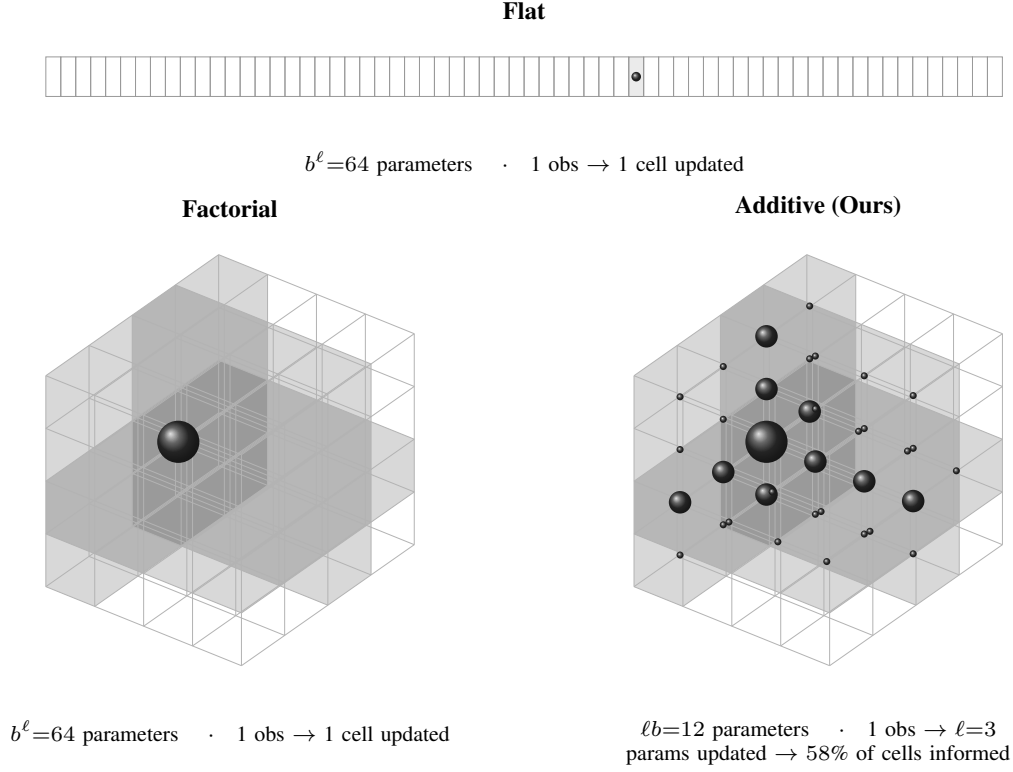
\begin{figure*}[t]
\centering
\resizebox{\textwidth}{!}{%
\begin{tikzpicture}

\def\bb{4}              
\def\numlev{3}          
\def\cs{0.6}            
\def\dotratio{0.4}      
\def\obsA{2}            
\def\obsB{1}            
\def\obsC{3}            
\def\isoanglea{35}      
\def\isoangleb{22}      
\def\panelgap{2.5}      
\def\blobrad{0.09}      

\pgfmathtruncatemacro{\lastidx}{\bb - 1}
\pgfmathtruncatemacro{\totalcells}{\bb^\numlev}
\pgfmathtruncatemacro{\lastcell}{\totalcells - 1}
\pgfmathtruncatemacro{\obsFlat}{\obsA*\bb*\bb + \obsB*\bb + \obsC}
\pgfmathtruncatemacro{\uninformed}{(\bb-1)^\numlev}
\pgfmathtruncatemacro{\informed}{\totalcells - \uninformed}
\pgfmathsetmacro{\informedpct}{round(\informed/\totalcells*100)}

\pgfmathsetmacro{\dotrad}{\dotratio * \cs}

\pgfmathsetmacro{\cubew}{\bb * (cos(\isoanglea) + cos(\isoangleb)) * \cs}
\pgfmathsetmacro{\cubeh}{\bb * \cs + \bb * (sin(\isoanglea) + sin(\isoangleb)) * \cs}

\pgfmathsetmacro{\cubeLeftX}{0}
\pgfmathsetmacro{\cubeRightX}{\cubew + \panelgap}

\pgfmathsetmacro{\visLeft}{\cubeLeftX - \bb * cos(\isoangleb) * \cs}
\pgfmathsetmacro{\visRight}{\cubeRightX + \bb * cos(\isoanglea) * \cs}
\pgfmathsetmacro{\visWidth}{\visRight - \visLeft}

\pgfmathsetmacro{\sqsize}{\visWidth / \totalcells}
\pgfmathsetmacro{\flatheight}{0.45}  
\pgfmathsetmacro{\flatdotrad}{\dotratio * \sqsize * 0.8}
\pgfmathsetmacro{\flatY}{\cubeh + 1.8}  

\pgfmathsetmacro{\labelpad}{0.5}

\newcommand{\isoproject}[3]{%
  \pgfmathsetmacro{\isox}{(#1) * cos(\isoanglea) * \cs - (#3) * cos(\isoangleb) * \cs}%
  \pgfmathsetmacro{\isoy}{(#2) * \cs + (#1) * sin(\isoanglea) * \cs + (#3) * sin(\isoangleb) * \cs}%
}

\newcommand{\fillslab}[7]{%
  \isoproject{#1}{#3}{#5} \coordinate (f0) at (\isox,\isoy);
  \isoproject{#2}{#3}{#5} \coordinate (f1) at (\isox,\isoy);
  \isoproject{#2}{#3}{#6} \coordinate (f2) at (\isox,\isoy);
  \isoproject{#1}{#3}{#6} \coordinate (f3) at (\isox,\isoy);
  \fill[black, opacity=#7] (f0)--(f1)--(f2)--(f3)--cycle;
  \isoproject{#1}{#4}{#5} \coordinate (f0) at (\isox,\isoy);
  \isoproject{#2}{#4}{#5} \coordinate (f1) at (\isox,\isoy);
  \isoproject{#2}{#4}{#6} \coordinate (f2) at (\isox,\isoy);
  \isoproject{#1}{#4}{#6} \coordinate (f3) at (\isox,\isoy);
  \fill[black, opacity=#7] (f0)--(f1)--(f2)--(f3)--cycle;
  \isoproject{#1}{#3}{#5} \coordinate (f0) at (\isox,\isoy);
  \isoproject{#1}{#4}{#5} \coordinate (f1) at (\isox,\isoy);
  \isoproject{#1}{#4}{#6} \coordinate (f2) at (\isox,\isoy);
  \isoproject{#1}{#3}{#6} \coordinate (f3) at (\isox,\isoy);
  \fill[black, opacity=#7] (f0)--(f1)--(f2)--(f3)--cycle;
  \isoproject{#2}{#3}{#5} \coordinate (f0) at (\isox,\isoy);
  \isoproject{#2}{#4}{#5} \coordinate (f1) at (\isox,\isoy);
  \isoproject{#2}{#4}{#6} \coordinate (f2) at (\isox,\isoy);
  \isoproject{#2}{#3}{#6} \coordinate (f3) at (\isox,\isoy);
  \fill[black, opacity=#7] (f0)--(f1)--(f2)--(f3)--cycle;
  \isoproject{#1}{#3}{#5} \coordinate (f0) at (\isox,\isoy);
  \isoproject{#2}{#3}{#5} \coordinate (f1) at (\isox,\isoy);
  \isoproject{#2}{#4}{#5} \coordinate (f2) at (\isox,\isoy);
  \isoproject{#1}{#4}{#5} \coordinate (f3) at (\isox,\isoy);
  \fill[black, opacity=#7] (f0)--(f1)--(f2)--(f3)--cycle;
  \isoproject{#1}{#3}{#6} \coordinate (f0) at (\isox,\isoy);
  \isoproject{#2}{#3}{#6} \coordinate (f1) at (\isox,\isoy);
  \isoproject{#2}{#4}{#6} \coordinate (f2) at (\isox,\isoy);
  \isoproject{#1}{#4}{#6} \coordinate (f3) at (\isox,\isoy);
  \fill[black, opacity=#7] (f0)--(f1)--(f2)--(f3)--cycle;
}

\newcommand{\isogrid}{%
  \foreach \gb in {0,...,\bb} {
    \foreach \gc in {0,...,\bb} {
      \isoproject{0}{\gb}{\gc}
      \pgfmathsetmacro{\sx}{\isox}\pgfmathsetmacro{\sy}{\isoy}
      \isoproject{\bb}{\gb}{\gc}
      \draw[gray!60, very thin] (\sx, \sy) -- (\isox, \isoy);
    }
  }
  \foreach \ga in {0,...,\bb} {
    \foreach \gc in {0,...,\bb} {
      \isoproject{\ga}{0}{\gc}
      \pgfmathsetmacro{\sx}{\isox}\pgfmathsetmacro{\sy}{\isoy}
      \isoproject{\ga}{\bb}{\gc}
      \draw[gray!60, very thin] (\sx, \sy) -- (\isox, \isoy);
    }
  }
  \foreach \ga in {0,...,\bb} {
    \foreach \gb in {0,...,\bb} {
      \isoproject{\ga}{\gb}{0}
      \pgfmathsetmacro{\sx}{\isox}\pgfmathsetmacro{\sy}{\isoy}
      \isoproject{\ga}{\gb}{\bb}
      \draw[gray!60, very thin] (\sx, \sy) -- (\isox, \isoy);
    }
  }
}

\begin{scope}[shift={(\visLeft cm, \flatY)}]
  \node[font=\footnotesize\bfseries, anchor=south]
    at ({\visWidth/2}, {\flatheight + 0.3}) {Flat};

  \foreach \i in {0,...,\lastcell} {
    \pgfmathtruncatemacro{\flatobs}{\i == \obsFlat ? 1 : 0}
    \draw[gray!80, ultra thin]
      ({\i*\sqsize}, 0) rectangle ++({\sqsize}, {\flatheight});
    \ifnum\flatobs=1
      \fill[black, opacity=0.08]
        ({\i*\sqsize}, 0) rectangle ++({\sqsize}, {\flatheight});
      \pgfmathsetmacro{\cx}{\i*\sqsize + \sqsize/2}
      \pgfmathsetmacro{\cy}{\flatheight/2}
      \shade[ball color=black!80] (\cx, \cy) circle (\flatdotrad);
    \fi
  }

  \node[font=\scriptsize, anchor=north, text width=5cm, align=center]
    at ({\visWidth/2}, {-\labelpad})
    {$b^\ell{=}64$ parameters \quad$\cdot$\quad
     1 obs $\to$ 1 cell updated};
\end{scope}

\begin{scope}[shift={(\cubeLeftX cm, 0)}]
  \isoproject{\bb/2}{0}{\bb/2}
  \node[font=\footnotesize\bfseries, anchor=south]
    at (\isox, {\cubeh + 0.3}) {Factorial};

  \fillslab{\obsA}{\obsA+1}{0}{\bb}{0}{\bb}{0.08}
  \fillslab{0}{\bb}{\obsB}{\obsB+1}{0}{\bb}{0.08}
  \fillslab{0}{\bb}{0}{\bb}{\obsC}{\obsC+1}{0.08}

  \isogrid

  \foreach \a in {\lastidx,...,0} {
    \foreach \c in {\lastidx,...,0} {
      \foreach \b in {0,...,\lastidx} {
        \isoproject{\a+0.5}{\b+0.5}{\c+0.5}
        \pgfmathtruncatemacro{\cubeobs}{\a*\bb*\bb + \b*\bb + \c}
        \ifnum\cubeobs=\obsFlat
          \shade[ball color=black!80] (\isox, \isoy) circle (\dotrad);
        \fi
      }
    }
  }

  \isoproject{\bb/2}{0}{\bb/2}
  \node[font=\scriptsize, anchor=north, text width=5cm, align=center]
    at (\isox, {-\labelpad})
    {$b^\ell{=}64$ parameters \quad$\cdot$\quad
     1 obs $\to$ 1 cell updated};
\end{scope}

\begin{scope}[shift={(\cubeRightX cm, 0)}]
  \isoproject{\bb/2}{0}{\bb/2}
  \node[font=\footnotesize\bfseries, anchor=south]
    at (\isox, {\cubeh + 0.3}) {Additive (Ours)};

  \fillslab{\obsA}{\obsA+1}{0}{\bb}{0}{\bb}{0.08}
  \fillslab{0}{\bb}{\obsB}{\obsB+1}{0}{\bb}{0.08}
  \fillslab{0}{\bb}{0}{\bb}{\obsC}{\obsC+1}{0.08}

  \isogrid

  \foreach \a in {\lastidx,...,0} {
    \foreach \c in {\lastidx,...,0} {
      \foreach \b in {0,...,\lastidx} {
        \isoproject{\a+0.5}{\b+0.5}{\c+0.5}
        \pgfmathtruncatemacro{\sa}{equal(\a, \obsA)}
        \pgfmathtruncatemacro{\sb}{equal(\b, \obsB)}
        \pgfmathtruncatemacro{\sc}{equal(\c, \obsC)}
        \pgfmathtruncatemacro{\shared}{\sa + \sb + \sc}
        \pgfmathsetmacro{\impactrad}{%
          \shared > 0 ?
            \dotrad * 2 * \shared * \shared / (\numlev * (\shared + \numlev))
          : 0}
        \ifdim\impactrad pt>0pt
          \shade[ball color=black!80] (\isox, \isoy) circle (\impactrad);
        \fi
      }
    }
  }

  \isoproject{\bb/2}{0}{\bb/2}
  \node[font=\scriptsize, anchor=north, text width=5cm, align=center]
    at (\isox, {-\labelpad})
    {$\ell b{=}12$ parameters \quad$\cdot$\quad
     1 obs $\to$ $\ell{=}\numlev$ params updated
     $\to$ \pgfmathprintnumber[precision=0]{\informedpct}\% of cells informed};
\end{scope}

\end{tikzpicture}
}

\caption{Partitioning the context space creates nonlinear reward
models from linear per-cell learners, but finer partitions fragment
the data: each cell sees fewer observations. All three schemes
realize $b^\ell{=}64$ cells ($b{=}4$, $\ell{=}3$).
\textbf{Top/Bottom left:} Flat and factorial parameterizations
learn each cell from its own data alone.
\textbf{Bottom right:} Additive RQ achieves the same $64$-cell
granularity with $\ell b{=}12$ shared parameters---each observation
updates $\ell{=}3$ centroids, pooling information across cells.
Fine partition without data starvation.}
\label{fig:system}
\end{figure*}

\begin{figure*}[t]
\centering
\includegraphics[width=\textwidth]{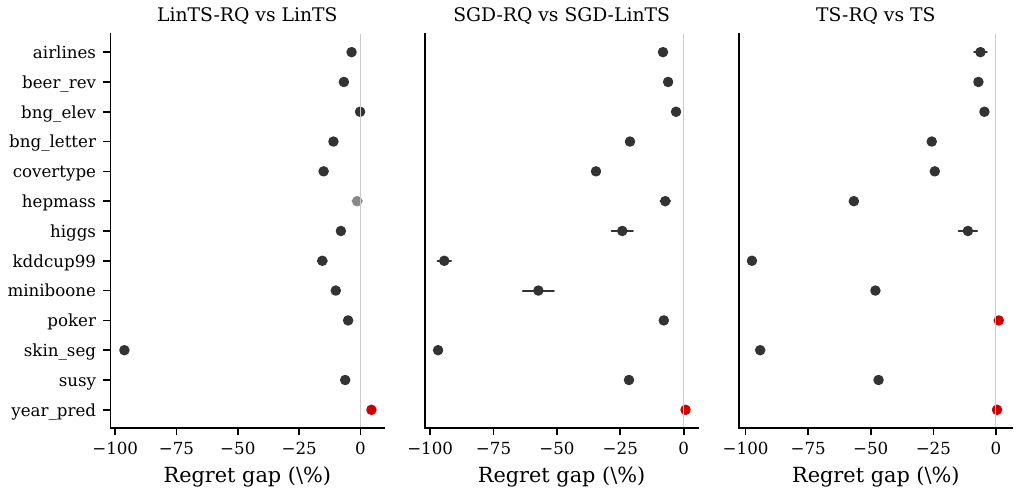}
\caption{Regret gap (\%) of each RQ method relative to its within-class
baseline on 13 datasets (30 seeds, 95\% CIs). Negative = RQ method has
lower regret.
Black: significant win (CI excludes zero).
Gray: not significant.
Red: significant loss.
\textbf{Left:} LinTS-RQ vs LinTS wins on 11/13 datasets; the
sole loss (year\_prediction, $d{=}90$, $K{=}84$) is explained by the
data sufficiency condition $N/(bKd^2) \ll 1$
(Section~\ref{sec:exp-limitation}).
\textbf{Center:} SGD-LinTS-RQ vs SGD-LinTS wins on 12/13.
\textbf{Right:} TS-RQ vs TS wins on 11/13.}
\label{fig:main_results}
\end{figure*}

\begin{figure}[t]
\centering
\includegraphics[width=0.85\columnwidth]{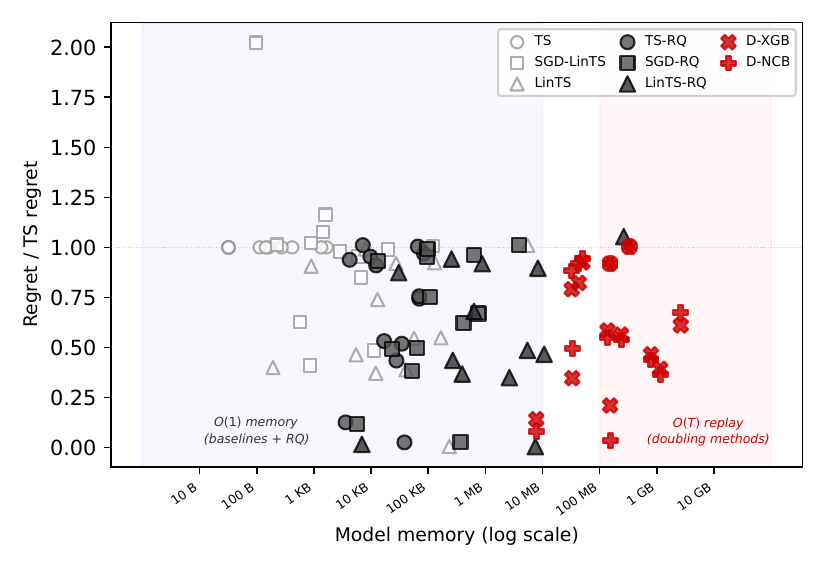}
\caption{Memory vs.\ quality across 13 datasets.
Each point is one (method, dataset) pair.
Hollow gray: linear baselines ($O(1)$ memory, high regret).
Solid gray: RQ methods ($O(1)$ memory, low regret).
Red: doubling methods ($O(T)$ replay buffers, low regret).
RQ methods match tree/neural regret at $1000\times$ less memory.
The vertical drop from hollow to solid at the same $x$-position
shows the regret gain from the RQ codebook.}
\label{fig:pareto}
\end{figure}

\section{Method}\label{sec:method}

\noindent\textbf{Setting.}\label{sec:setting}
We consider the standard stochastic contextual bandit setting~\cite{lattimore2020bandit}: at each
round~$t$, the agent observes a context $x_t \in \mathbb{R}^d$, selects
an arm $a_t \in [K]$, and receives a reward
$r_t \in \{0, 1\}$.
The goal is to minimize cumulative regret
$R(T) = \sum_{t=1}^T [r^*(x_t) - r_t]$, where $r^*$ is the oracle
reward.

We assume access to a pool of $n_{\mathrm{rq}}$ \emph{unlabeled}
contexts---separate from the bandit evaluation pool---on which we
train the RQ codebook via sequential $k$-means.
We train the codebook once and freeze it before the bandit starts.
This is a standard assumption in settings where unlabeled contexts are
abundant (user profiles, item catalogs, sensor readings) but rewards
require costly sequential interaction.
The codebook encodes each context via greedy $k$-means assignment:
\begin{equation}\label{eq:rq_encode}
    k_i = \arg\min_{k \in [b]} \|x^{(i-1)} - c_k^{(i)}\|^2, \qquad
    x^{(i)} = x^{(i-1)} - c_{k_i}^{(i)}, \qquad x^{(0)} = x,
\end{equation}
producing a code sequence $z = [k_1, \ldots, k_\ell]$ that
partitions the context space, and context residuals
$x^{(0)}, x^{(1)}, \ldots, x^{(\ell-1)}$ that serve as per-level
feature inputs to the bandit model below.

\smallskip
\noindent\textbf{Gradient-boosted bandit over RQ levels.}\label{sec:algorithms}
All methods share the architecture in Algorithm~\ref{alg:rqbandit}.
Each RQ level operates on the context residual left by shallower
levels, and boosts~\cite{friedman2001greedy} the reward prediction:
level~$0$ predicts the raw reward; level~$i > 0$ predicts the
\emph{pseudo-target}---the gap between the observed reward and the
cumulative prediction of levels $0, \ldots, i{-}1$.
A damping factor $\eta \in (0,1]$ controls each level's
contribution.

\begin{algorithm}[t]
\caption{RQ Bandit with Shadow Promotion}
\label{alg:rqbandit}
\begin{algorithmic}[1]
\Require Codebook $\mathcal{C}^{(1)}, \ldots, \mathcal{C}^{(\ell_{\max})}$;
         model array $\mathcal{M}[i, c, a]$ for each level $i \in [\ell_{\max}]$,
         centroid $c \in [b]$, arm $a \in [K]$;
         clipping function $\sigma(z) = \mathrm{clip}(z + 0.5, 0, 1)$;
         damping $\eta$; initial depth $\ell_0$;
         significance level $\alpha$
\State $\ell_{\mathrm{act}} \leftarrow \ell_0$;
       initialize shadow at level $\ell_0{+}1$
       \label{alg:init}
\For{$t = 1, \ldots, T$}
    \State \textbf{Encode:} observe $x_t$; compute codes
           $(k_1, \ldots, k_{\ell_{\max}})$ and context residuals
           $(x_t^{(0)}, x_t^{(1)}, \ldots)$ via~\eqref{eq:rq_encode}
           \label{alg:encode}
    \State \textbf{Score:}
           $s_a \leftarrow \sum_{i=1}^{\ell_{\mathrm{act}}}
           \eta^{i-1} \cdot
           \mathcal{M}[i,\, k_i,\, a].\mathrm{sample}(x_t^{(i-1)})$
           for each arm $a$
           \label{alg:score}
    \State \textbf{Select:}
           $a_t \leftarrow \arg\max_a s_a$;
           observe reward $r_t \in \{0,1\}$
           \label{alg:select}
    \State \textbf{Targets (residual cascade):}
           $\hat{y}_0 \leftarrow 0$
           \label{alg:targets}
    \For{$i = 1, \ldots, \ell_{\mathrm{act}}$}
        \State $\hat{\mu}_{i} \leftarrow
               \mathcal{M}[i,\, k_i,\, a_t].\mathrm{mean}(x_t^{(i-1)})$
               \Comment{level-$i$ prediction}
        \State $\tilde{r}_{i} \leftarrow r_t -
               \sigma(\hat{y}_{i-1})$
               \Comment{residual target}
        \State $\hat{y}_i \leftarrow \hat{y}_{i-1} +
               \eta^{i-1} \hat{\mu}_{i}$
               \Comment{accumulate}
    \EndFor
    \State \textbf{Update:}
           $\mathcal{M}[i,\, k_i,\, a_t].\mathrm{update}(x_t^{(i-1)},
           \tilde{r}_i)$
           for $i = 1, \ldots, \ell_{\mathrm{act}}$
           \label{alg:update}
    \State \textbf{Shadow:}
           $\mathcal{M}[\ell_{\mathrm{act}}{+}1,\, k_{\ell_{\mathrm{act}}+1},\, a_t].\mathrm{update}(x_t^{(\ell_{\mathrm{act}})}, \tilde{r}_{\ell_{\mathrm{act}}+1})$
           \label{alg:shadow}
    \If{$t \in \{T_0, 2T_0, 4T_0, \ldots\}$ \textbf{and}
        $\ell_{\mathrm{act}} < \ell_{\max}$}
        \If{Freedman test (Theorem~\ref{thm:shadow},
            Appendix~\ref{app:promotion_proofs}): shadow
            significantly better at level
            $\alpha / (n_c \cdot (\ell_{\max}{-}1))$, where $n_c$ is the number of checks}
            \State $\ell_{\mathrm{act}} \mathrel{+}= 1$;
                   spawn new shadow
                   \label{alg:promote}
        \EndIf
    \EndIf
\EndFor
\end{algorithmic}
\end{algorithm}

\noindent The three instantiations differ only in the per-level model
class~$\mathcal{M}$ (Table~\ref{tab:variants}).

\begin{table}[t]
\centering
\caption{Model variants sharing Algorithm~\ref{alg:rqbandit}, differing only in per-cell learner~$\mathcal{M}$
(full pseudocode in Appendix~\ref{app:model_variants}).}
\label{tab:variants}
\begin{tabular}{@{}l l l l@{}}
\toprule
\textbf{Method} & $\mathcal{M}$ & \textbf{Features used} & \textbf{Cost/round} \\
\midrule
TS-RQ & Count/sum & Index $k_i$ only & $O(\ell K)$ \\
SGD-LinTS-RQ & Diagonal linear & Index $k_i$ + $[x^{(i-1)}; 1] \in \mathbb{R}^{d+1}$ & $O(\ell d K)$ \\
LinTS-RQ   & Full-covariance LinTS & Index $k_i$ + $[x^{(i-1)}; 1] \in \mathbb{R}^{d+1}$ & $O(\ell d^2 K)$ \\
\bottomrule
\end{tabular}
\end{table}

\noindent For LinTS-RQ, pseudo-targets at deeper levels have
smaller magnitude and fewer effective samples; we describe
the per-level regularization and numerical scaling in
Section~\ref{sec:exp-setup}.

\smallskip
\noindent\textbf{Shadow promotion.}\label{sec:drq}
Algorithm~\ref{alg:rqbandit} adapts depth $\ell$ online via a shadow
level that trains alongside the active model but does not influence
arm selection (line~\ref{alg:shadow}).
At pre-committed check times (line~\ref{alg:promote}), a statistical
test compares the shadow's prediction accuracy against the active
model over a trailing window.
If the shadow wins, it is promoted and a new shadow spawned;
otherwise the active depth is unchanged.
We analyze the false-promotion guarantee in
Section~\ref{sec:analysis-shadow}.

\section{Analysis}\label{sec:analysis}

We analyze Algorithm~\ref{alg:rqbandit} along three axes:
asymptotic regret under base-learner convergence properties,
the effect of depth on misspecification, and false-promotion
control for the shadow mechanism.

\subsection{Asymptotic regret}\label{sec:analysis-regret}

\begin{definition}[Additive misspecification]\label{def:delta}
The \emph{additive misspecification} at depth~$\ell$ is
\begin{equation}\label{eq:delta}
    \Delta_\ell(\mathcal{M}) = \inf_{f_1, \ldots, f_\ell \in \mathcal{M}}
    \mathbb{E}_{x}\!\left[
        \max_a \left| r^*(x, a) -
        \sigma\!\left(\sum_{i=1}^{\ell}
        f_i(k_i(x), a)\right) \right|
    \right].
\end{equation}
\end{definition}

\noindent We require three properties of the per-cell base
learner~$\mathcal{M}$, stated informally here (formal
definitions and verification for all three variants in
Appendix~\ref{app:base_learner_axioms}):
\begin{enumerate}
\item \emph{Eventual boundedness}: score means do not
      diverge.
\item \emph{Exploration}: if an arm rarely visits a
      centroid, its score variance remains positive,
      ensuring continued exploration.
\item \emph{Consistency}: when shallower levels stabilize,
      each level's predictions converge to the best
      predictor in its function class.
\end{enumerate}

\begin{theorem}[Asymptotic regret]\label{thm:asymptotic}
For any base learner~$\mathcal{M}$ satisfying
Properties~1--3, let $\hat{r}_\infty$ denote the converged
prediction of the sequential residual-fitting procedure
(Appendix~\ref{app:fixed_depth_proofs}).
Then
\begin{equation}\label{eq:asymptotic_regret}
    \limsup_{T \to \infty}\;
    \frac{R(T, \ell)}{T}
    \;\le\;
    2\,\mathbb{E}_{x}\!\left[\max_a
    \left|r^*(x, a) - \hat{r}_\infty(x, a)\right|
    \right].
\end{equation}
Since $\hat{r}_\infty$ is one feasible additive predictor,
the right-hand side equals
$2\,\Delta_\ell(\mathcal{M})$ when the stagewise and joint
optima coincide; in general it is controlled by the
stagewise approximation error.
\end{theorem}

\noindent The key technical step is an inductive convergence
proof across boosting levels: each level's targets depend on
all shallower levels' estimates, creating a coupled system
not covered by single-model analysis.
We establish Properties~1--3 for TS-RQ, SGD-LinTS-RQ, and
LinTS-RQ in Appendix~\ref{app:base_learner_axioms}; the
full proof appears in
Appendix~\ref{app:fixed_depth_proofs}.
The verification follows the standard approach in the TS
literature~\cite{agrawal2013thompson}; a fully rigorous
treatment of the strong law under adaptive Thompson Sampling
with misspecification remains open in general, as noted
by~\cite{agrawal2012analysis}.
The bound says that average regret per round depends on the
stagewise approximation error, not directly on the horizon $T$.
When that error is zero, the method achieves sublinear regret;
under misspecification, regret may grow linearly at the residual
error rate. This is not the optimal $O(\sqrt{T})$ rate that LinTS
achieves under correct specification; the tradeoff is a richer
non-linear reward model at fixed memory.

\subsection{Misspecification and depth}
\label{sec:analysis-depth}

The additive parameterization resolves $b^\ell$ effective cells
with only $\ell b$ parameters per arm, and each observation
updates $\ell$ centroid models---so exploration cost grows
linearly in~$\ell$ while resolution grows exponentially.
Since deeper models contain shallower ones
($f_{\ell+1} \equiv 0$),
$\Delta_{\ell+1} \le \Delta_\ell$ trivially; the question
is the \emph{rate} of decay.

We can now separate the two mechanisms that determine this
rate.  For one arm and the identity-link readout, let $S_\ell$
be the saturated space of square-integrable functions of the
prefix code $z_{1:\ell}$, and let
$A_\ell=\{\sum_{j\leq\ell}g_j(z_j)\}$ be the compact additive
space used by RQ.  Write
$\delta_\ell=\operatorname{dist}(r^*,A_\ell)$,
$\delta_\ell^{\rm sat}=\operatorname{dist}(r^*,S_\ell)$,
and
$\iota_\ell=\operatorname{dist}(P_{S_\ell}r^*,A_\ell)$.
Let $A_\infty=\overline{\bigcup_\ell A_\ell}$,
$\delta_\infty=\operatorname{dist}(r^*,A_\infty)$, and
$\tau_\ell=\operatorname{dist}(P_{A_\infty}r^*,A_\ell)$.

\begin{theorem}[Depth-rate decomposition]\label{thm:depth-rate}
If $r^*$ is $L$-Lipschitz and
$\mathrm{MSE}_\ell=\mathbb{E}\|x-\hat x_\ell\|^2$ is the
depth-$\ell$ RQ reconstruction error, then
\begin{align}
\delta_\ell^2
&=(\delta_\ell^{\rm sat})^2+\iota_\ell^2,
&
(\delta_\ell^{\rm sat})^2
&\leq L^2\mathrm{MSE}_\ell,
&
\iota_\ell
&\leq\delta_\infty+\tau_\ell .
\label{eq:depth-decomposition}
\end{align}
Consequently,
\begin{equation}
\delta_\ell^2-\delta_\infty^2
\leq L^2\mathrm{MSE}_\ell
   +2\delta_\infty\tau_\ell+\tau_\ell^2.
\label{eq:depth-excess}
\end{equation}
\end{theorem}

\begin{corollary}[Conditional geometric decay]\label{cor:depth-rate}
If $\mathrm{MSE}_\ell\leq M_0\rho_x^\ell$ and
$\tau_\ell\leq C_\tau\rho_\tau^\ell$, with
$\rho_x,\rho_\tau\in(0,1)$, then
\begin{equation}
\delta_\ell
\leq \delta_\infty
 +L\sqrt{M_0}\,\rho_x^{\ell/2}
 +C_\tau\rho_\tau^\ell
=\delta_\infty+O(\rho^\ell),
\end{equation}
where $\rho=\max\{\sqrt{\rho_x},\rho_\tau\}<1$.
\end{corollary}

\noindent The proof is in
Appendix~\ref{app:fixed_depth_proofs}.  Reconstruction controls
the saturated $b^\ell$-cell benchmark, while $\tau_\ell$
measures the additional interaction tail imposed by the compact
$O(\ell b)$ additive readout.  Thus geometric reconstruction
decay alone is insufficient; geometric additive-tail decay is
the explicit additional condition.  For multiple arms,
$\mathbb{E}\max_a|e_a|\leq
(\sum_a\mathbb{E}e_a^2)^{1/2}$, so the per-arm result transfers
to Definition~\ref{def:delta} up to a $\sqrt K$ factor.  A
Lipschitz output link similarly transfers a pre-link score bound
after multiplication by its Lipschitz constant.

Figure~\ref{fig:decay_summary} is consistent with the conditional
geometric law: fitting $a+b\rho^\ell$ to regret rates at
depths~$1$--$5$ yields median $R^2=0.99$ across 39
(dataset, method) pairs.  This finite-depth evidence does not
establish the assumptions of Corollary~\ref{cor:depth-rate}
universally.  Within-cell reward variance contracts with median
per-level ratio $0.83$ (Appendix~\ref{app:depth_ablation}).

\begin{figure}[t]
\centering
\includegraphics[width=0.85\columnwidth]{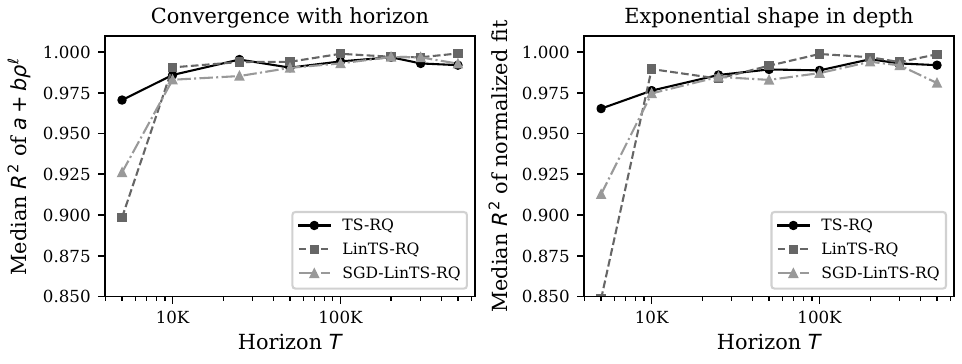}
\caption{Empirical behavior consistent with the conditional
geometric depth-rate law.
\textbf{Left:} Median $R^2$ of the exponential fit
$a + b\rho^\ell$ improves with horizon~$T$.
\textbf{Right:} Normalized regret decay
$(\Delta_\ell - \Delta_5)/(\Delta_1 - \Delta_5)$
at terminal~$T$ for each dataset; curves
cluster along exponential trajectories.}
\label{fig:decay_summary}
\end{figure}

\subsection{Shadow false-promotion control}
\label{sec:analysis-shadow}

\begin{theorem}[Shadow false-promotion control]
\label{thm:shadow}
The probability that the shadow is promoted despite having no
lower expected squared prediction error than the active model,
under the active policy's data distribution over the tested
window, is at most~$\alpha$ over the full horizon.
\end{theorem}

\noindent We test the paired squared-error differences via Freedman's
martingale inequality~\cite{freedman1975tail}, which applies
without requiring i.i.d.\ observations.
Given a finite planning horizon $T$ (or an upper bound on the
deployment length), the number of check times per depth is
$n_c = \lfloor\log_2(T/T_0)\rfloor$, known before the first
observation. The per-test significance level is Bonferroni-corrected
over all $n_c \cdot (\ell_{\max}{-}1)$ comparisons, giving total
false-promotion probability $\le \alpha$.
Proof in Appendix~\ref{app:promotion_proofs}.

\section{Experiments}\label{sec:experiments}

\subsection{Setup}\label{sec:exp-setup}
Following the supervised-to-bandit
protocol~\cite{riquelme2018deep}, we convert datasets
by assigning reward~$1$ for the correct class and~$0$
otherwise.
We evaluate on 13 tabular datasets from OpenML and UCI with
$N$ ranging from $130\text{K}$ to $11\text{M}$, feature dimension
$d \in [3, 90]$, and $K \in [2, 104]$ arms
(Table~\ref{tab:datasets} in Appendix~\ref{app:datasets}).
Following the codebook training protocol (Section~\ref{sec:setting}),
$n_{\mathrm{rq}}{=}50\text{K}$ contexts are held out for codebook
training; the bandit runs on the remaining pool with contexts drawn uniformly
at random without replacement, each seen once,
so the horizon $T = N - n_{\mathrm{rq}}$.
All methods use $b{=}16$ centroids per level, initial depth
$\ell_0{=}2$, $\ell_{\max}{=}8$, $\nu{=}0.1$, $\eta{=}0.5$,
$\lambda{=}1.0$, shadow $\alpha{=}0.05$, $W{=}1000$ (trailing window),
$T_0 = 50 \cdot b \cdot K$,
and 30 bandit seeds per configuration.
For LinTS-RQ, we scale the per-level regularization as
$\lambda_i = \lambda \cdot 0.25^{i-1}$ and add numerical jitter
$10^{i-1} \times 10^{-6}$ for Cholesky stability.
We fix all hyperparameters across all datasets---no per-dataset tuning.
Code is available at
\url{https://github.com/atavory/rqb}.
Each seed initializes all random number generators (arm selection,
posterior sampling) deterministically.

We compare six methods: three baselines (TS, SGD-LinTS, LinTS)
and their RQ variants with adaptive depth (TS-RQ, SGD-LinTS-RQ,
LinTS-RQ). We additionally compare against three nonlinear baselines:
Doubling-XGB (XGBoost: 100 trees, max depth~6, learning rate~0.1,
subsample~0.8, colsample\_bytree~0.8, $\epsilon{=}0.1$
greedy, doubling retrain schedule),
Doubling-NeuralCB (2-layer MLP, 128 hidden units, ReLU activations,
Adam optimizer with learning rate $10^{-3}$, batch size~256,
doubling retrain schedule, LinTS head),
and TabNet-LinTS~\cite{arik2021tabnet} (unsupervised
TabNet pre-training with default hyperparameters:
$N_{\mathrm{steps}}{=}3$, $N_a{=}N_d{=}8$, $\gamma{=}1.3$,
$\lambda_{\mathrm{sparse}}{=}10^{-3}$; encoder frozen on the same 50K
holdout, LinTS on learned features).
All representation-based methods (RQ, TabNet, Neural-Linear) use the
same 50K unlabeled holdout; the baselines (TS, LinTS, SGD-LinTS)
operate without any representation layer by design---they have no
mechanism to consume unlabeled data---and evaluate on the same
$T$-context pool, ensuring identical evaluation conditions.
We fix all hyperparameters---both for RQ methods and for the nonlinear
baselines (XGBoost, NeuralCB, TabNet)---across all 13
datasets with no per-dataset tuning.

We report regret \emph{rate} (cumulative regret divided by $T$)
throughout, making results length-invariant across datasets.
Win/loss calls use 95\% CIs from 30 seeds; a method ``wins'' only
when the CI excludes zero.

\subsection{Results}\label{sec:exp-main}

\noindent\textbf{Memory--quality tradeoff.}
Figure~\ref{fig:pareto} places all methods on a memory--regret
scatter, and Table \ref{tab:memory} gives the details. Regret is normalized by the context-free oracle rate
($1 - \max_a P(Y{=}a)$), so 1.0 corresponds to majority-class
guessing. Values above 1.0 occur when squared-loss exploration
produces overconfident wrong predictions on skewed datasets.
Doubling-XGB and Doubling-NeuralCB match LinTS-RQ
regret but require $O(T)$ replay memory (Table~\ref{tab:memory}).

\begin{table}[t]
\centering
\caption{Model memory at shadow-promoted depths ($b{=}16$). XGB replay
buffer grows $O(T)$; all RQ methods are fixed $O(1)$.}
\label{tab:memory}
\begin{tabular}{@{}l r r r@{}}
\toprule
& \textbf{TS-RQ} & \textbf{LinTS-RQ} & \textbf{XGB replay} \\
\midrule
higgs ($T{=}5\text{M}$) & 12\,KB & 620\,KB & 1.1\,GB \\
susy ($T{=}5\text{M}$) & 22\,KB & 176\,KB & 800\,MB \\
covertype ($T{=}530\text{K}$) & 60\,KB & 10\,MB & 51\,MB \\
\bottomrule
\end{tabular}
\end{table}
\smallskip

\noindent\textbf{Within-family comparisons.}
Figure~\ref{fig:main_results} shows the regret gap of each RQ method
relative to its within-class baseline, and Table \ref{tab:main_results} gives the details.
LinTS-RQ wins significantly on 10 of 13 datasets, with gaps
up to $96\%$ (skin\_segmentation);
SGD-LinTS-RQ wins on 12 of 13;
TS-RQ wins on 11 of 13 versus context-free TS.
The common failure case is year\_prediction ($d{=}90$,
$K{=}84$), where all three RQ variants show small losses
($+0.4\%$ to $+4.5\%$); we discuss this in the data
sufficiency analysis below.

\smallskip
\noindent\textbf{Nonlinear baselines.}\label{sec:exp-xgb}
LinTS-RQ achieves regret competitive with Doubling-XGB (7W/4L/2T).
The key differentiator is memory: at $T{=}5\text{M}$,
LinTS-RQ uses 620\,KB vs.\ XGB's 1.1\,GB---a gap that grows linearly
with $T$. Against TabNet-LinTS (same 50K holdout), LinTS-RQ wins on
11 of 13 datasets (11W/2L), with losses only on hepmass
($+0.5\%$) and year\_prediction ($+1.7\%$), showing that
unsupervised $k$-means produces a partition that matches or
exceeds TabNet's neural embedding, while the RQ codebook
trains in seconds via $k$-means versus minutes for TabNet
pretraining.
Per-step latency is also low: TS-RQ runs at
$50$--$750\,\mu\text{s}$/step across all datasets
(Appendix~\ref{app:timing}).

\begin{table}[t]
\centering
\caption{Win/loss/tie (W/L/T) summary across 13 datasets
(30 seeds, 95\% CIs). Negative gap = RQ method has lower regret.}
\label{tab:main_results}
\begin{tabular}{@{}l c c c@{}}
\toprule
\textbf{Comparison} & \textbf{W/L/T} & \textbf{Best win} & \textbf{Worst loss} \\
\midrule
LinTS-RQ vs LinTS & 10W/1L/2T & skin\_seg $-96\%$ & year\_pred $+4.5\%$ \\
SGD-RQ vs SGD-LinTS & 12W/1L & skin\_seg $-97\%$ & year\_pred $+0.7\%$ \\
TS-RQ vs TS & 11W/2L & kddcup99 $-97\%$ & year\_pred $+0.4\%$ \\
\bottomrule
\end{tabular}
\end{table}

\smallskip

\noindent\textbf{Data sufficiency.}
\label{sec:exp-limitation}
LinTS-RQ maintains a $d \times d$ covariance matrix for each of the $bK$
centroid-arm pairs, so the total parameter count scales as $bKd^2$.
The ratio $N / (bKd^2)$ measures contexts per parameter and is computable
before any bandit interaction (Figure~\ref{fig:sufficiency}).
On 12 of 13 datasets this ratio exceeds $1.1$ and LinTS-RQ wins;
on year\_prediction ($d{=}90$, $K{=}84$, ratio${=}0.04$), LinTS-RQ
degrades even at depth~1.
TS-RQ needs only $O(1)$ parameters per cell (a count and a reward sum),
so it is unaffected: it matches the TS baseline on year\_prediction
($93.0\%$ vs $93.4\%$).
The failure is structural, not fixable by tuning: LinTS-RQ degrades
with depth ($96.3\%$ at $\ell{=}1$, $97.4\%$ at $\ell{=}4$) as
each level fragments the already-sparse data further, while
TS-RQ stays flat at $93.0\%$ across all depths.
On synthetic data, LinTS-RQ wins at all ratios
$\ge 14$ (Appendix~\ref{app:synthetic}).

\begin{figure}[t!]
\centering
\includegraphics[width=0.85\columnwidth]{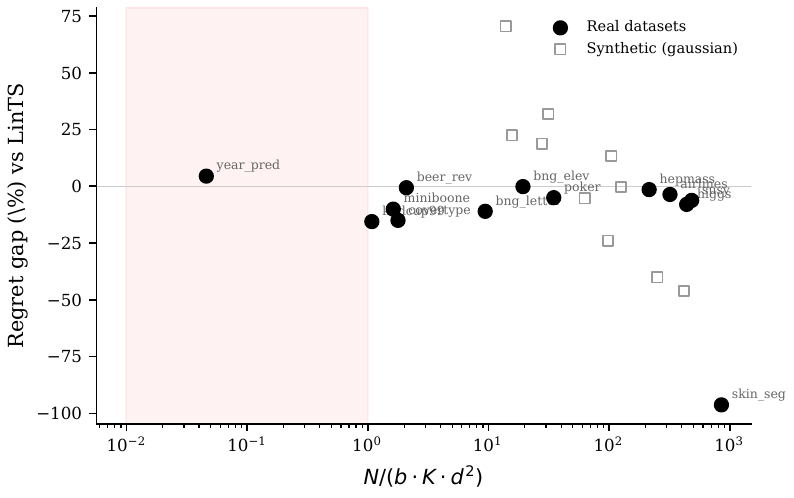}
\caption{Data sufficiency diagnostic: regret gap of LinTS-RQ
vs.\ LinTS as a function of $N/(bKd^2)$.
All real datasets with ratio~$> 1$ (outside the shaded region)
are wins or ties; year\_prediction (ratio~$= 0.04$) is the
sole loss. Synthetic Gaussian data (gray squares) confirm
the threshold (Appendix~\ref{app:synthetic}).}
\label{fig:sufficiency}
\end{figure}

\smallskip

\noindent\textbf{Effective rank.}
\label{sec:exp-correlation}
On datasets with effective rank
($\mathrm{tr}(\Sigma)/\lambda_{\max}(\Sigma)$) below $1.5$
(skin\_segmentation, miniboone),
SGD-LinTS performs \emph{worse than context-free TS}: the diagonal
covariance approximation spreads its $d$ parameters uniformly across
coordinates, wasting $d{-}1$ of them on noise when most variance lies
on a single axis.
TS-RQ avoids this failure mode entirely---the codebook geometry
captures dominant directions through centroid placement, without
estimating any covariance---and achieves $2.6\%$ and $14.7\%$
normalized regret on these two datasets
(Appendix~\ref{app:correlation}).

\smallskip

\noindent\textbf{Depth adaptation.}\label{sec:exp-depth}
Table~\ref{tab:depth_ablation} compares TS-RQ at fixed depths
against shadow-promoted adaptive depth.
On covertype, shadow promotion discovers depth~4 and gains $3.4$
percentage points over fixed $\ell{=}2$.
On higgs, the benefit beyond depth~2 is marginal, and the adaptive
method correctly settles at depth~3 with matching regret:
the marginal gain from depth~4 falls within the noise floor
of the Freedman test, so the shadow is never promoted.

\begin{table}[t]
\centering
\caption{TS-RQ regret rate (\%) at fixed vs.\ adaptive depth.}
\label{tab:depth_ablation}
\begin{tabular}{@{}l r r r@{}}
\toprule
& \textbf{$\ell{=}1$} & \textbf{$\ell{=}2$}
& \textbf{Adaptive} \\
\midrule
covertype  & 48.3\% & 42.2\% & 38.8\% (depth 4) \\
higgs      & 45.4\% & 43.6\% & 43.7\% (depth 3) \\
\bottomrule
\end{tabular}
\end{table}

\section{Related Work}\label{sec:related}

\noindent\textbf{Linear and GLM bandits.}
LinTS~\cite{agrawal2013thompson} and
LinUCB~\cite{li2010contextual,abbasi2011improved,dani2008stochastic}
achieve $O(d\sqrt{KT\log T})$ regret under linear rewards.
The misspecification bound~\cite{lattimore2020learning} gives $\Delta \cdot T$
linear regret under misspecification---the basis for our
asymptotic analysis (Appendix~\ref{app:fixed_depth_proofs}).

\smallskip
\noindent\textbf{Neural and tree-based bandits.}
NeuralUCB~\cite{zhou2020neural} and Neural
TS~\cite{zhang2021neural} require per-step gradient and
inverse-Hessian computation; efficient approximations
(EE-Net~\cite{ban2022eenet}, Neural-Linear~\cite{riquelme2018deep})
reduce to LinTS on last-layer features.
TETS~\cite{nilsson2024tets} combines XGBoost with TS but must refit
on accumulated history, incurring $O(T^2)$ cumulative cost.
Our offline codebook avoids both: no gradients, no refitting.
\smallskip
\noindent\textbf{Partition-based bandits.}
The X-armed bandit framework~\cite{bubeck2011x} and the Zooming
algorithm~\cite{kleinberg2008multi,slivkins2014contextual} adaptively
refine partitions to balance resolution against sample complexity.
Nonparametric histogram bandits~\cite{rigollet2010nonparametric,perchet2013multi}
achieve regret $O(T^{(d+2)/(d+1)})$ with flat partitions.
Tile coding~\cite{sutton2018reinforcement} uses additive partitions
in RL but with uniform tilings; our data-driven RQ codebook adapts
to $P(X)$. The additive parameterization uses $\ell b$ parameters
for $b^\ell$ cells---a potentially exponential compression.

\smallskip
\noindent\textbf{Model selection and quantization.}
CORRAL~\cite{agarwal2017corralling} and Regret
Balancing~\cite{pacchiano2020regret} run multiple base learners in
parallel. Our shadow promotion monitors one additional depth level
sequentially without influencing arm selection (Theorem~\ref{thm:shadow}).
Product quantization~\cite{jegou2010product}, residual
quantization~\cite{chen2010approximate}, and VQ-VAE~\cite{van2017neural}
compress representations; we show the same structure compresses
\emph{exploration} ($b^\ell$ cells from $\ell b$ parameters).
Self-supervised methods (SCARF~\cite{bahri2022scarf},
VIME~\cite{yao2020vime}, SubTab~\cite{ucar2021subtab},
TabNet~\cite{arik2021tabnet}) learn representations from unlabeled
data; our codebook achieves this via unsupervised $k$-means.

\section{Conclusion}\label{sec:conclusion}

We repurposed RQ codebooks as a representation layer for
contextual bandits, enabling gradient-boosted reward models with
fixed $O(1)$ memory. RQ variants improve on most datasets at
orders-of-magnitude less memory than replay-based methods, while
TS-RQ remains robust when LinTS-RQ's covariance starves.
Extensions to continuous rewards, online codebook adaptation,
and alternative per-level features (PCA, neural-linear) are
straightforward.

\section{Limitations}\label{sec:limitations}

The data sufficiency condition $N/(bKd^2) \gg 1$ is the sharpest
limitation: on year\_prediction ($d{=}90$, $K{=}84$, ratio~$0.04$),
LinTS-RQ degrades even at depth~1; a $p$-dimensional projection
($p \ll d$) would relax this to $N/(bKp^2)$.
The linear variants use squared loss for binary rewards (P\'olya-Gamma
augmentation could eliminate this).
Shadow promotion provides false-promotion control, though
post-promotion performance depends on the new policy's data
distribution (Remark~\ref{rem:covariate}).
The fixed-depth asymptotic analysis
(Appendix~\ref{app:fixed_depth_proofs}) is conditional on
base-learner convergence properties; the appendix makes these
assumptions explicit and verifies them for our model classes.
The codebook is trained via unsupervised $k$-means, which
minimizes reconstruction error rather than reward prediction;
a reward-aware partition (e.g., CART splits) could reduce the
misspecification $\Delta_\ell$ directly, at the cost of requiring
labeled data for codebook construction.
The codebook requires an offline holdout; in cold-start settings
it may degrade.
As with all bandit methods, reward-signal biases may be amplified; our method neither exacerbates nor mitigates this.

\section{Broader Impact}\label{sec:broader_impact}

This work introduces a general-purpose bandit representation and
does not target a specific application domain.
Standard fairness and privacy concerns for bandit algorithms
apply; we do not foresee risks specific to residual quantization.

\section*{Acknowledgments and Disclosures}

This work received no external funding. The authors declare no
competing interests.

\bibliographystyle{plainnat}
\bibliography{references}

\newpage
\appendix

\section{Base Learner Details}\label{app:model_variants}

This appendix specifies the per-cell operations for each model
class~$\mathcal{M}$ in Algorithm~\ref{alg:rqbandit}.
All three variants share the same outer loop; they differ only in how
each cell $\mathcal{M}[i, c, a]$ maintains its posterior and produces
predictions.
Each cell exposes three operations:
\texttt{sample}($x'$) returns a Thompson-sampling draw (used for arm
selection in the \textbf{Score} step);
\texttt{mean}($x'$) returns the posterior mean prediction (used to
build the cumulative prediction $\hat{y}_i$ in the \textbf{Targets}
step);
\texttt{update}($x'$, $\tilde{r}$) incorporates a new observation.
Here $x' = x^{(i-1)}$ is the context residual
from~\eqref{eq:rq_encode} and $\tilde{r}$ is the pseudo-target from
Algorithm~\ref{alg:rqbandit}, line~\ref{alg:targets}.
For the linear variants (SGD-LinTS-RQ, LinTS-RQ), a constant
intercept feature is appended: $x' \leftarrow [x^{(i-1)}; 1]$,
so the effective per-centroid dimension is $d{+}1$.
This ensures $\|\phi_i(x)\| \ge 1$
(Assumption~\ref{asm:features}), which is needed for the
exploration guarantee even when deep-level residuals are small.

\begin{algorithm}[H]
\caption{TS-RQ cell operations (count/sum)}
\label{alg:ts_rq}
\begin{algorithmic}[1]
\Require Regularization $\lambda$; exploration scale $\nu$
\State \textbf{State:} count $n \leftarrow 0$;\; sum $s \leftarrow 0$
\Function{sample}{$x'$} \Comment{$x'$ is ignored}
    \State $\hat{\mu} \leftarrow s\, / (\lambda + n)$
    \State \Return sample from
           $\mathcal{N}\!\left(\hat{\mu},\;
           \nu^2 / (\lambda + n)\right)$
\EndFunction
\Function{mean}{$x'$} \Comment{$x'$ is ignored}
    \State \Return $s\, / (\lambda + n)$
\EndFunction
\Function{update}{$x',\, \tilde{r}$} \Comment{$x'$ is ignored}
    \State $n \leftarrow n + 1$;\; $s \leftarrow s + \tilde{r}$
\EndFunction
\end{algorithmic}
\end{algorithm}

\begin{algorithm}[H]
\caption{SGD-LinTS-RQ cell operations (diagonal linear)}
\label{alg:sgd_lints_rq}
\begin{algorithmic}[1]
\Require Regularization $\lambda$; exploration scale $\nu$
\State \textbf{State:} diagonal precision
       $p \in \mathbb{R}^{d+1} \leftarrow \lambda \mathbf{1}$;\;
       sufficient statistic $f \in \mathbb{R}^{d+1} \leftarrow \mathbf{0}$
\Comment{$x' \leftarrow [x'; 1]$ (intercept appended)}
\Function{sample}{$x'$}
    \State $\hat{\theta} \leftarrow f \oslash p$
        \Comment{element-wise division}
    \State $\tilde{\theta} \sim \mathcal{N}\!\left(\hat{\theta},\;
           \nu^2\, \mathrm{diag}(p)^{-1}\right)$
    \State \Return $\tilde{\theta}^\top x'$
\EndFunction
\Function{mean}{$x'$}
    \State \Return $(f \oslash p)^\top x'$
\EndFunction
\Function{update}{$x',\, \tilde{r}$}
    \State $p \leftarrow p + x' \odot x'$
        \Comment{element-wise}
    \State $f \leftarrow f + \tilde{r} \cdot x'$
\EndFunction
\end{algorithmic}
\end{algorithm}

\begin{algorithm}[H]
\caption{LinTS-RQ cell operations (full-covariance LinTS)}
\label{alg:lints_rq}
\begin{algorithmic}[1]
\Require Regularization $\lambda_i = \lambda \cdot 0.25^{i-1}$;
         exploration scale $\nu$
\State \textbf{State:} inverse precision
       $B^{-1} \in \mathbb{R}^{(d+1) \times (d+1)}
       \leftarrow \lambda_i^{-1} I$;\;
       $f \in \mathbb{R}^{d+1} \leftarrow \mathbf{0}$
\Comment{$x' \leftarrow [x'; 1]$ (intercept appended)}
\Function{sample}{$x'$}
    \State $\hat{\theta} \leftarrow B^{-1} f$
    \State $\tilde{\theta} \sim \mathcal{N}\!\left(\hat{\theta},\,
           \nu^2 B^{-1}\right)$
    \State \Return $\tilde{\theta}^\top x'$
\EndFunction
\Function{mean}{$x'$}
    \State \Return $(B^{-1} f)^\top x'$
\EndFunction
\Function{update}{$x',\, \tilde{r}$}
    \Comment{Sherman--Morrison on $B^{-1}$}
    \State $v \leftarrow B^{-1} x'$
    \State $B^{-1} \leftarrow B^{-1}
           - v\, v^\top / (1 + x'^\top v)$
    \State $f \leftarrow f + \tilde{r} \cdot x'$
\EndFunction
\end{algorithmic}
\end{algorithm}

\section{Base Learner Properties}\label{app:base_learner_axioms}

Algorithm~\ref{alg:rqbandit} is parameterized by a per-cell base
learner~$\mathcal{M}$.
Theorem~\ref{thm:asymptotic} (\S\ref{sec:analysis}) establishes
asymptotic regret under three properties
of~$\mathcal{M}$; the full proof is in
Appendix~\ref{app:fixed_depth_proofs}.
We state these properties below and establish them for each variant.
We use the notation from Appendix~\ref{app:fixed_depth_proofs}:
$\phi_i(x)$ is the level-$i$ feature vector,
$\mu_a(x_t)$ the score mean~\eqref{eq:score_mean},
$\mathrm{Var}(s_a \mid x_t, \mathcal{F}_{t-1})$ the score
variance~\eqref{eq:score_var},
$T_{a,c,i}(t)$ the visit count for arm~$a$ at centroid~$c$
and level~$i$, and
$\tilde{r}_{i,t}$ the per-level target~\eqref{eq:level_target}.

\begin{description}
\item[Property~1 (Eventual boundedness).]
Score means are eventually bounded almost surely:
\begin{equation}\label{eq:axiom1}
    \exists_{t_0, B}\; \forall_{a,\, t > t_0}\;
    |\mu_a(x_t)| \le B
    \qquad \text{a.s.}
\end{equation}

\item[Property~2 (Starvation $\Rightarrow$ positive variance).]
If arm~$a$ visits centroid~$c$ at level~$i$ at most $M$ times
($T_{a,c,i}(t) \le M$ for all~$t$), then
\begin{equation}\label{eq:axiom2}
    \mathrm{Var}(s_a \mid x_t, \mathcal{F}_{t-1})
    \ge \sigma_{\min}^2 > 0
    \qquad \forall\, t \text{ with } k_i(x_t) = c,
\end{equation}
where $\sigma_{\min}^2$ depends on $M, \lambda, \nu, C$
but not on~$t$.

\item[Property~3 (Robust consistency).]
Suppose all centroids at level~$i$ are visited infinitely often
by arm~$a$: $T_{a,c,i}(t) \to \infty$ for all~$c$ with
$P(k_i(x) = c) > 0$.
Suppose the targets have vanishing perturbation:
$\tilde{r}_{i,t} = \tilde{r}_{i,t}^{\infty} + \delta_t$ with
$\delta_t \to 0$ a.s., where $\tilde{r}_{i,t}^{\infty}$ is the
stationary target under converged shallower levels.
Then predictions converge:
\begin{equation}\label{eq:axiom3}
    \phi_i(x)^\top \hat{\theta}_{a,i,t}
    \;\to\; \phi_i(x)^\top \theta_{a,i}^*
    \qquad \text{a.s.\ for all } x \text{ in the support},
\end{equation}
where $\theta_{a,i}^*$ is the best predictor of
$\tilde{r}_{i}^{\infty}$ from $\phi_i(x)$ in the function class
of~$\mathcal{M}$.
\end{description}

\subsection{Verification for LinTS-RQ}\label{sec:axioms_lints}

LinTS-RQ maintains a full-covariance precision
$V_{a,i,t} = \lambda_i I + \sum_{s \le t:\, a_s=a}
\phi_i(x_s)\,\phi_i(x_s)^\top$
and posterior mean
$\hat{\theta}_{a,i,t} = V_{a,i,t}^{-1}
\sum_{s:a_s=a} \phi_i(x_s)\,\tilde{r}_{i,s}$.

\smallskip
\noindent\textbf{Property~1.}
For any arm~$a$, level~$i$, and centroid~$c$:
if $T_{a,c,i}(t) \le M$ for all~$t$, the estimate
receives at most $M$ updates with bounded targets
($|\tilde{r}| \le 1$), so
$\|\hat{\theta}_{a,i,c}\| \le MC/\lambda$
(since $\|V^{-1}\| \le 1/\lambda$ and
$\|\sum \phi \tilde{r}\| \le MC$), giving
$|\hat{\theta}_{a,i,c}^\top \phi_i(x)|
\le MC^2/\lambda$, deterministically bounded.
If $T_{a,c,i}(t) \to \infty$, the strong law gives
$\hat{\theta}_{a,i,c} \to \theta_{a,i,c}^*$, which is finite.
Since the number of arm/centroid combinations is finite,
score means are eventually bounded.

\smallskip
\noindent\textbf{Property~2.}
If $T_{a,c,i}(t) \le M$ for all~$t$, the centroid-$c$
block of $V_{a,i,t}$ has at most $M$ rank-one updates,
giving the time-uniform bound
$V_{a,i,t}\big|_c \preceq (\lambda + MC^2) I$.
When $k_i(x_t) = c$:
\begin{align}
    \mathrm{Var}(s_a \mid x_t, \mathcal{F}_{t-1})
    &\ge \nu^2\, \phi_i(x_t)^\top V_{a,i,t}^{-1} \phi_i(x_t)
    \ge \frac{\nu^2}{\lambda + MC^2}
    \eqqcolon \sigma_{\min}^2 > 0.
\end{align}

\smallskip
\noindent\textbf{Property~3.}
Suppose $T_{a,c,i}(t) \to \infty$ for all centroids~$c$
and targets satisfy
$\tilde{r}_{i,s} = \tilde{r}_{i,s}^{\infty} + \delta_s$
with $\delta_s \to 0$.
The ridge estimator decomposes as
\begin{align}
    \hat{\theta}_{a,i,t}
    &= \underbrace{V_{a,i,t}^{-1}
       \sum_{s:a_s=a} \phi_i(x_s)\,
       \tilde{r}_{i,s}^{\infty}}_{\to\, \theta_{a,i}^*}
    \;+\; \underbrace{V_{a,i,t}^{-1}
       \sum_{s:a_s=a} \phi_i(x_s)\,
       \delta_s}_{\text{transient bias}}.
    \label{eq:lints_decomp}
\end{align}
The first term converges by standard ridge consistency
(stationary targets, $T_{a,c,i} \to \infty$).
For the transient bias, apply Cauchy--Schwarz under
$\|v\|_{M} = \sqrt{v^\top M\, v}$ with $M = V_{a,i,t}^{-1}$:
\begin{align}\label{eq:transient_cs_axiom}
    \biggl|\phi_i(x)^\top V_{a,i,t}^{-1}
    \sum_{s:a_s=a} \phi_i(x_s)\,\delta_s\biggr|
    \le
    \|\phi_i(x)\|_{V_{a,i,t}^{-1}}
    \cdot
    \biggl\|\sum_{s:a_s=a}
    \phi_i(x_s)\,\delta_s\biggr\|_{V_{a,i,t}^{-1}}.
\end{align}

\noindent\emph{First factor.}
Write
$\hat{\Sigma}_t = \frac{1}{T_a(t)}
\sum_{s:a_s=a} \phi_i(x_s)\,\phi_i(x_s)^\top$
and $\Sigma = \mathbb{E}[\phi_i(x)\phi_i(x)^\top \mid a]$.
Since $V_{a,i,t} = T_a(t)\bigl(
\tfrac{\lambda}{T_a(t)}\,I + \hat{\Sigma}_t\bigr)$:
\begin{align}
    \|\phi_i(x)\|_{V_{a,i,t}^{-1}}^2
    &= \phi_i(x)^\top V_{a,i,t}^{-1}\,\phi_i(x)
    \nonumber \\
    &= \frac{1}{T_a(t)}\;
       \phi_i(x)^\top
       \!\left(\tfrac{\lambda}{T_a(t)}\,I
       + \hat{\Sigma}_t\right)^{\!-1}
       \phi_i(x)
    \nonumber \\
    &\xrightarrow{\text{a.s.}}\;
       \frac{1}{T_a(t)}\;
       \phi_i(x)^\top \Sigma^{-1} \phi_i(x)
    \nonumber \\
    &\le \frac{c}{T_a(t)}
    \qquad \text{eventually},
\end{align}
where the third line uses
$\hat{\Sigma}_t \xrightarrow{\text{a.s.}} \Sigma$
(the strong law),
and $c = 2C^2 / \min_c \lambda_{\min}(\Sigma_c) < \infty$.
Each centroid block~$\Sigma_c$ is positive definite
under a mild nondegeneracy condition: the conditional
distribution of $\phi_i(x)$ given $k_i(x) = c$ has
full-rank covariance.
In practice, the intercept feature ensures at least
one nonzero direction, and the RQ residuals within
each centroid span the remaining coordinates when
$T_{a,c}$ is large enough.
If $\Sigma_c$ is rank-deficient, the regularized
estimator still converges in prediction space: since
$\phi_i(x)$ for centroid~$c$ lies in the column space
of~$\Sigma_c$, the prediction
$\phi_i(x)^\top \hat{\theta}$ converges regardless.
Also, $\|\phi_i(x)\| \le C$ by
Assumption~\ref{asm:features}.

\smallskip
\noindent\emph{Second factor.}
Write $\Phi \in \mathbb{R}^{bp \times T_a(t)}$ for the matrix
whose columns are $\phi_i(x_s)$ and
$\Delta \in \mathbb{R}^{T_a(t)}$ for the vector of
drifts~$\delta_s$.
By the push-through identity
$(\lambda I + AB)^{-1}A = A(\lambda I + BA)^{-1}$:
\begin{align}
    \biggl\|\sum_{s:a_s=a}
    \phi_i(x_s)\,\delta_s\biggr\|_{V_{a,i,t}^{-1}}^2
    &= \Delta^\top \Phi^\top
       (\lambda I + \Phi\Phi^\top)^{-1} \Phi\,\Delta
    \nonumber \\
    &= \Delta^\top
       \underbrace{(\lambda I + \Phi^\top\Phi)^{-1}
       \Phi^\top\Phi}_{\preceq\, I}\,\Delta
    \nonumber \\
    &\le \|\Delta\|^2
    = \sum_{s:a_s=a} \delta_s^2.
\end{align}

\noindent\emph{Combining.}
\begin{align}
    \biggl|\phi_i(x)^\top V_{a,i,t}^{-1}
    \sum_{s:a_s=a} \phi_i(x_s)\,\delta_s\biggr|^2
    \le \frac{c}{T_a(t)} \sum_{s:a_s=a} \delta_s^2.
\end{align}
Since $\delta_s \to 0$ and $T_a(t) \to \infty$,
the Ces\`{a}ro mean
$\frac{1}{T_a(t)}\sum_{s:a_s=a} \delta_s^2 \to 0$,
so the transient bias prediction vanishes.
\qed

\subsection{Verification for SGD-LinTS-RQ}\label{sec:axioms_sgd}

SGD-LinTS-RQ maintains a diagonal precision
$p \in \mathbb{R}^d$ with
$p_k = \lambda + \sum_{s:a_s=a} x'^2_{s,k}$
and sufficient statistic
$f_k = \sum_{s:a_s=a} \tilde{r}_s\, x'_{s,k}$,
yielding the coordinate-wise estimator
$\hat{\theta}_k = f_k / p_k$.
This is $d$ independent scalar regressions, one per
feature coordinate.

\smallskip
\noindent\textbf{Property~1.}
Same structure as LinTS-RQ: bounded updates when
starved, SLLN when explored, applied coordinate-wise.

\smallskip
\noindent\textbf{Property~2.}
If $T_{a,c,i}(t) \le M$, then
$p_{a,i,k,t} \le \lambda + MC^2$ for every
coordinate~$k$, giving
\begin{align}
    \mathrm{Var}(s_a \mid x_t, \mathcal{F}_{t-1})
    &= \sum_{j=1}^{\ell} \nu^2 \sum_{k=1}^{d+1}
    \frac{(\phi_{j,k})^2}{p_{a,j,k,t}}
    \ge \frac{\nu^2 \|\phi_i(x_t)\|^2}{\lambda + MC^2}
    \ge \frac{\nu^2}{\lambda + MC^2} > 0
\end{align}
by Assumption~\ref{asm:features}, where the sum is
over $d{+}1$ coordinates (including the intercept that
ensures $\|\phi_i(x)\| \ge 1$ even when context
residuals are small at deep levels).

\smallskip
\noindent\textbf{Property~3.}
By the strong law applied to each coordinate independently:
\begin{align}
    \hat{\theta}_{a,i,k,t} \to \theta^*_{a,i,k}
    \coloneqq \frac{\mathbb{E}[\tilde{r}_{i}^{\infty}
    \cdot x'_k]}{\mathbb{E}[x'^2_k]},
\end{align}
the coordinate-wise optimal coefficient.
The transient bias at coordinate~$k$ satisfies
\begin{align}
    |f^{\mathrm{bias}}_k / p_k|
    \le C\, \frac{1}{T_a(t)} \sum_{s:a_s=a} |\delta_s|
    \to 0
\end{align}
by the Ces\`{a}ro argument.
Note that $\theta^*_{a,i}$ is the coordinate-wise optimal,
which may differ from the full-linear optimal when features
are correlated; consequently
$\Delta_\ell(\mathcal{M})$ for SGD-LinTS-RQ is evaluated
over the coordinate-wise linear function class.
\qed

\subsection{Verification for TS-RQ}\label{sec:axioms_ts}

TS-RQ is the scalar special case: the per-cell model
ignores features and maintains only a count~$n$ and
sum~$s$ (Algorithm~\ref{alg:ts_rq}).
The estimate $\hat{\mu} = s / (\lambda + n)$ is a scalar
ridge regression with constant feature~$1$.
All three properties follow from the LinTS-RQ checks
with $d = 1$ and $\phi_i(x) = 1$.
\qed

\begin{remark}[Adaptive sampling]\label{rem:adaptive_sampling}
The checks above invoke the strong law of large numbers
for per-centroid sample averages.
Under Thompson Sampling, the data arriving at each centroid is not
i.i.d.: arm selection depends on the posterior, which depends on
all past observations.
However, each centroid's samples form a \emph{conditionally}
unbiased sequence:
conditioned on $a_t = a$ and $k_i(x_t) = c$, the target
$\tilde{r}_{i,t}$ has the same conditional expectation regardless
of how $a_t$ was selected.
The key requirement is that every centroid is visited infinitely
often (Lemma~\ref{lem:explore_converge}(a)), which ensures the
denominators grow.
Given infinite visits, the Ces\`{a}ro convergence argument
applies to the prediction error regardless of the adaptive
selection mechanism, since the per-step bias $\delta_s$ vanishes
and the averaging is over a growing number of bounded terms.
This is the standard approach in the TS
literature~\cite{agrawal2013thompson,agrawal2012analysis};
a fully formal treatment under adaptive sampling remains open
for misspecified TS in general.
\end{remark}

\section{Fixed-Depth Regret Analysis}\label{app:fixed_depth_proofs}

\noindent This appendix develops the full fixed-depth analysis
for Algorithm~\ref{alg:rqbandit}, including the asymptotic
regret theorem and the depth-rate decomposition of
Theorem~\ref{thm:depth-rate}.

\begin{assumption}\label{asm:features}
All per-level feature vectors are bounded:
$1 \le \|\phi_i(x)\| \le C$ for all~$i, x$.
\end{assumption}

\noindent The lower bound $\|\phi_i(x)\| \ge 1$ is
ensured by appending a constant intercept feature to
each $\phi_i(x)$.

\noindent\textbf{Definition~\ref{def:delta}} (Additive
misspecification, restated from \S\ref{sec:analysis}).
The additive misspecification $\Delta_\ell(\mathcal{M})$
is defined in~\eqref{eq:delta}.
Since deeper models contain shallower ones
($f_{\ell+1} \equiv 0$),
$\Delta_{\ell+1} \le \Delta_\ell$.

\noindent\textbf{Theorem~\ref{thm:asymptotic}} (Asymptotic
regret, restated from \S\ref{sec:analysis}).
\textit{For any base learner~$\mathcal{M}$ satisfying
Properties~1--3 (Appendix~\ref{app:base_learner_axioms}),
the asymptotic regret rate satisfies~\eqref{eq:asymptotic_regret},
where $\hat{r}_\infty$ is the converged prediction of the
sequential residual-fitting limit
(Lemma~\ref{lem:explore_converge}(b)).
Since $\hat{r}_\infty$ is one feasible additive
predictor, the right-hand side is
$\ge 2\,\Delta_\ell(\mathcal{M})$,
with equality when the stagewise and joint optima coincide.}

\begin{lemma}[Exploration and convergence]\label{lem:explore_converge}
For any base learner~$\mathcal{M}$ satisfying Properties~1--3,
for each arm~$a$ and level~$i \in \{1, \ldots, \ell\}$:
\begin{enumerate}
\item[\emph{(a)}] Every arm visits every centroid at
level~$i$ infinitely often:
$T_{a,c}(T) \to \infty$ a.s.\ for all centroids~$c$
with $P(k_i(x) = c) > 0$.
\item[\emph{(b)}] The per-level predictions converge:
$\phi_i(x)^\top \hat{\theta}_{a,i,t}
\to \phi_i(x)^\top \theta_{a,i}^*$ for all~$x$ in the
support.
\end{enumerate}
\end{lemma}

\noindent The proof proceeds by induction on the level~$i$.
Each level's targets depend on all shallower levels'
estimates, creating a coupled non-stationary system.
Properties~1--3
(Appendix~\ref{app:base_learner_axioms}) abstract this
coupling for three distinct base learners.
The per-level posterior can be viewed as a block-diagonal
approximation of the joint posterior analyzed in
Appendix~\ref{app:analytical_variants_proof}.

\smallskip
\noindent\textbf{Notation.}
Each feature vector $\phi_j(x)$ is supported on a single
centroid block at level~$j$ (placing the context residual
$x^{(j-1)}$ at centroid~$k_j(x)$ and zeros elsewhere).
The TS draw, score, and selection for arm~$a$ at context~$x_t$ are
\begin{align}
\forall_a \tilde{\theta}_{a,j}
&\sim \mathcal{M}[j,\, k_j,\, a].\mathrm{sample}(x_t^{(j-1)}),
\label{eq:ts_draw}
\\
\forall_a
s_a &= \sum_{j=1}^{\ell}
\tilde{\theta}_{a,j}^\top \phi_j(x_t),
\label{eq:score_def}
\\
a_t &= \arg\max_a s_a =
\arg\max_a \sigma(s_a)
\end{align}
independently across levels; the last equality holds
since $\sigma$ is monotone.
The \emph{score mean} and conditional variance are
\begin{align}
\mu_a(x_t) &= \mathbb{E}[s_a \mid x_t, \mathcal{F}_{t-1}]
= \sum_{j=1}^{\ell}
\hat{\theta}_{a,j,t}^\top \phi_j(x_t),
\label{eq:score_mean}
\\
\mathrm{Var}(s_a \mid x_t, \mathcal{F}_{t-1})
&= \sum_{j=1}^{\ell} \nu^2\,
\phi_j(x_t)^\top V_{a,j,t}^{-1} \phi_j(x_t).
\label{eq:score_var}
\end{align}
The reward prediction is $\hat{y}_a = \sigma(\mu_a) \in [0,1]$,
and the per-level target (Algorithm~\ref{alg:rqbandit},
line~\ref{alg:targets}) is
\begin{equation}\label{eq:level_target}
    \tilde{r}_{j,t} = r_t - \sigma\!\left(
    \sum_{i=1}^{j-1}
    \hat{\theta}_{a_t,i,t}^\top \phi_i(x_t)\right) \in [-1, 1].
\end{equation}
For clarity, the damping factor $\eta$ from
Algorithm~\ref{alg:rqbandit} is set to~$1$ in the
analysis; absorbing $\eta^{j-1}$ into $\phi_j$ does
not affect the proof structure, though it scales the
effective variance at level~$j$ by $\eta^{2(j-1)}$,
which slows exploration at deeper levels when
$\eta < 1$. The asymptotic result is unaffected
(the lower bound $\sigma_{\min}^2$ remains positive),
but the pre-asymptotic exploration cost at deep
levels grows as $\eta^{-2(j-1)}$.

\begin{lemma}[Score means are eventually bounded]
\label{lem:score_bounded}
All arm score means are eventually bounded:
\begin{align}
   \exists_{t_0, B} \forall_{a,\, t > t_0} |\mu_a(x_t)| \le B
   \qquad \text{a.s.}
\end{align}
\end{lemma}

\begin{proof}
Immediate from Property~1
(Appendix~\ref{app:base_learner_axioms}).
\end{proof}

\begin{lemma}[Exploration and convergence, restated]
For any base learner~$\mathcal{M}$ satisfying Properties~1--3
(Appendix~\ref{app:base_learner_axioms}),
for each arm~$a$ and level~$i \in \{1, \ldots, \ell\}$:
\begin{enumerate}
\item[\emph{(a)}] Every arm visits every centroid at
level~$i$ infinitely often:
$T_{a,c}(T) \to \infty$ a.s.\ for all centroids~$c$
with $P(k_i(x) = c) > 0$.
\item[\emph{(b)}] The per-level predictions converge:
$\phi_i(x)^\top \hat{\theta}_{a,i,t}
\to \phi_i(x)^\top \theta_{a,i}^*$ for all~$x$ in the
support, where $\theta_{a,i}^*$ is the best linear
predictor for the clipped residual at level~$i$ given
converged shallower levels.
\end{enumerate}
\end{lemma}

\begin{proof}[Proof of Lemma~\ref{lem:explore_converge}]
The proof is by induction on the level $i$ and applies to
any base learner~$\mathcal{M}$ satisfying Properties~1--3.

\smallskip
\noindent\emph{Induction base ($i=1$).}

\smallskip
\noindent\emph{(a) Per-centroid exploration.}
Assume that for some $a, c$,
\begin{align}
T_{a,c}(T) &\le M < \infty,
\\
P(k_1(x) = c) &\eqqcolon p_c > 0. \label{eq:centroid_positive}
\end{align}
By Property~2, the score variance satisfies
\begin{align}
\mathrm{Var}(s_a \mid x_t, \mathcal{F}_{t-1})
\ge \sigma_{\min}^2 > 0
\qquad \forall\, t \text{ with } k_1(x_t) = c.
\end{align}
By Lemma~\ref{lem:score_bounded}, there exist
(a.s.)\ finite $B$ and~$t_0$ such that every arm's
score mean satisfies $|\mu_{a'}(x_t)| \le B$ for all
$t \ge t_0$.
For $t \ge t_0$ with $k_1(x_t) = c$, the
probability that arm~$a$ is selected satisfies
\begin{align}
P(a_t = a \mid x_t, \mathcal{F}_{t-1})
&\ge P(s_a > B)\,
\prod_{a' \ne a} P(s_{a'} \le B)
\nonumber \\
&\ge \bigl[1 - \Phi(2B/\sigma_{\min})\bigr]
\cdot ({\tfrac{1}{2}})^{K-1}
\eqqcolon p_{\min} > 0,
\end{align}
where the first factor uses anti-concentration of the
score distribution: all three variants use Gaussian
sampling (Algorithms~\ref{alg:ts_rq}--\ref{alg:lints_rq}),
so $s_a$ is Gaussian with variance
$\ge \sigma_{\min}^2$ and mean $\ge -B$, giving the
stated normal tail bound.
Each factor in the product is $\ge \tfrac{1}{2}$ because
$\mu_{a'} \le B$.
Since $k_1(x_t) = c$ occurs infinitely often
(i.i.d.\ contexts with $p_c > 0$) and each
occurrence satisfies
$P(a_t = a \mid x_t, \mathcal{F}_{t-1}) \ge p_{\min}$
(a uniform conditional lower bound, not independence),
the L\'{e}vy extension of the Borel--Cantelli lemma
gives $T_{a,c} \to \infty$,
contradicting~\eqref{eq:centroid_positive}.

\smallskip
\noindent\emph{(b) Prediction convergence.}
The targets $\tilde{r}_{1,t} = r_t - \sigma(0)$ are
stationary (no dependence on shallower levels), so
$\delta_t = 0$.
By~(a), $T_{a,c}(T) \to \infty$ for every centroid~$c$.
Property~3 (with $\delta_t = 0$) gives
$\phi_1(x)^\top \hat{\theta}_{a,1,t}
\to \phi_1(x)^\top \theta_{a,1}^*$.

\smallskip
\noindent\emph{Induction step ($i \ge 2$).}
Assume (a) and (b) hold for levels $1, \ldots, i{-}1$.

\smallskip
\noindent\emph{(a) Per-centroid exploration.}
Similar to the induction base: if $T_{a,c}(T) \le M$
for some centroid~$c$ at level~$i$, Property~2 gives
$\mathrm{Var}(s_a) \ge \sigma_{\min}^2 > 0$
when $k_i(x_t) = c$.
By Lemma~\ref{lem:score_bounded}, score means are
eventually bounded, so Gaussian anti-concentration and
Borel--Cantelli yield the contradiction
$T_{a,c} \to \infty$.

\smallskip
\noindent\emph{(b) Prediction convergence.}
The clipped residual target for level~$i$ is
\begin{equation}
    \tilde{r}_{i,t} = r_t -
    \sigma\!\left(\sum_{j=1}^{i-1}
    \hat{\theta}_{a,j,t}^\top \phi_j(x_t)\right).
\end{equation}
Since $\sigma$ is $1$-Lipschitz and levels
$1, \ldots, i{-}1$ have converged in prediction space
by the induction hypothesis~(b), the target drift
satisfies
\begin{equation}\label{eq:drift}
    |\tilde{r}_{i,t} - \tilde{r}_{i,t}^{\infty}|
    \le \sum_{j=1}^{i-1}
    \left|(\hat{\theta}_{a,j,t} - \theta_{a,j}^*)^\top
    \phi_j(x_t)\right|
    \eqqcolon \delta_t \to 0,
\end{equation}
where each summand vanishes by the induction hypothesis.
By~(a), $T_{a,c,i}(t) \to \infty$ for all centroids.
Property~3 (with $\delta_t \to 0$) gives
$\phi_i(x)^\top \hat{\theta}_{a,i,t}
\to \phi_i(x)^\top \theta_{a,i}^*$.

\end{proof}

\begin{proof}[Proof of Theorem~\ref{thm:asymptotic}
(Asymptotic regret for Algorithm~\ref{alg:rqbandit})]
By Lemma~\ref{lem:explore_converge}(b), for each arm~$a$,
the converged prediction is
\begin{align}\hat{r}_\infty(x, a) =
\sigma\!\left(\sum_{i=1}^{\ell}
\theta_{a,i}^{*\top} \phi_i(x)\right).
\end{align}
Define
\begin{align}\hat{a}(x) = \arg\max_a \hat{r}_\infty(x, a)
\end{align}
as the arm selected under the converged model.

\smallskip
\noindent\emph{Post-convergence regret.}
Once predictions have converged, TS selects $\hat{a}(x_t)$
at every round (the posterior variance vanishes;
model-level ties $\hat{r}_\infty(x, a) = \hat{r}_\infty(x, a')$
occur on a measure-zero set of contexts under continuous features,
and contribute at most $2\Delta_\ell$ regret per round regardless).
The per-round regret is
\begin{align}
    r^*(x_t, a_t^*) - r^*(x_t, \hat{a}(x_t))
    &= \bigl[r^*(x_t, a_t^*) - \hat{r}_\infty(x_t, a_t^*)\bigr]
    \nonumber \\
    &\quad - \bigl[r^*(x_t, \hat{a}) -
       \hat{r}_\infty(x_t, \hat{a})\bigr]
    \nonumber \\
    &\quad + \underbrace{\bigl[\hat{r}_\infty(x_t, a_t^*)
       - \hat{r}_\infty(x_t, \hat{a})\bigr]}_{\le\, 0},
\end{align}
where the last term is $\le 0$ because $\hat{a}$
maximizes $\hat{r}_\infty$.
Therefore
\begin{equation}
    r^*(x_t, a_t^*) - r^*(x_t, \hat{a}(x_t))
    \le 2 \max_a \left|r^*(x_t, a)
    - \hat{r}_\infty(x_t, a)\right|.
\end{equation}
Taking expectation over $x_t$:
\begin{equation}\label{eq:post_conv}
    \mathbb{E}\!\left[r^*(x_t, a_t^*)
    - r^*(x_t, \hat{a}(x_t))\right]
    \le 2\,\mathbb{E}_{x}\!\left[\max_a
    \left|r^*(x, a) - \hat{r}_\infty(x, a)\right|
    \right]
    \eqqcolon 2\,\bar{\Delta}_\ell(\mathcal{M}),
\end{equation}
where $\bar{\Delta}_\ell$ is the approximation error
of the sequential residual-fitting limit
$\hat{r}_\infty(x, a) = \sigma(\sum_i
\theta_{a,i}^{*\top} \phi_i(x))$.
Since $\hat{r}_\infty$ is one feasible additive
predictor, $\bar{\Delta}_\ell \ge \Delta_\ell$
(Definition~\ref{def:delta}), with equality when
the stagewise and joint optima coincide.

\smallskip
\noindent\emph{Exploration regret.}
Before convergence, TS may select suboptimal arms.
By Lemma~\ref{lem:explore_converge}(b),
$\hat{\theta}_{a,i,t} \to \theta_{a,i}^*$ for all
arms and levels (in prediction space), so the posterior
means converge.
Since the posterior variances vanish
($\nu^2 V_{a,i,t}^{-1} \to 0$ as $T_a(t) \to \infty$
by Lemma~\ref{lem:explore_converge}(a)),
the TS draws converge:
\begin{align}\tilde{\theta}_{a,i,t} \to \theta_{a,i}^*
\end{align}
a.s., and arm selection converges to $\hat{a}(x_t)$.
For contexts where the optimal arm gap
$\hat{r}_\infty(x, \hat{a}(x)) -
\max_{a \ne \hat{a}(x)} \hat{r}_\infty(x, a) > 0$,
the TS draws eventually resolve to the correct arm;
exact ties have zero probability under continuous
TS draws.
For any $\epsilon > 0$, there exists $T_0$ such that
for $t > T_0$, $a_t = \hat{a}(x_t)$ except on a set
of measure $< \epsilon$.
The exploration regret (rounds where
$a_t \ne \hat{a}(x_t)$) therefore satisfies
\begin{align}\sum_{t=1}^{T} \mathbf{1}[a_t \ne \hat{a}(x_t)] = o(T).
\end{align}
Since per-round regret is bounded by~$1$,
the cumulative exploration regret is $o(T)$.

\smallskip
\noindent\emph{Total.}
\begin{equation}
    \frac{R(T, \ell)}{T}
    = \underbrace{\frac{o(T)}{T}}_{\to\, 0}
    + \underbrace{2\,
      \bar{\Delta}_\ell(\mathcal{M})}_{\text{from~\eqref{eq:post_conv}}},
\end{equation}
giving $\limsup_{T \to \infty} R(T, \ell)/T
\le 2\,\bar{\Delta}_\ell(\mathcal{M})
= 2\,\mathbb{E}_x[\max_a |r^* - \hat{r}_\infty|]$.
\qedhere
\end{proof}

\subsection{Proof of the depth-rate result}

Fix one arm and let $\mathcal H=L^2(P)$. Because each RQ
symbol takes finitely many values, $A_\ell$ is a closed
finite-dimensional subspace of the closed saturated space
$S_\ell$. Denote their orthogonal projections by
$P_{A_\ell}$ and $P_{S_\ell}$.

\begin{proof}[Proof of Theorem~\ref{thm:depth-rate}]
Since $A_\ell\subseteq S_\ell$,
$P_{A_\ell}=P_{A_\ell}P_{S_\ell}$. We may therefore write
\begin{equation}
r^*-P_{A_\ell}r^*
=\bigl(r^*-P_{S_\ell}r^*\bigr)
 +\bigl(P_{S_\ell}r^*-P_{A_\ell}r^*\bigr).
\end{equation}
The first summand lies in $S_\ell^\perp$ and the second in
$S_\ell$. Pythagoras gives the first identity in
\eqref{eq:depth-decomposition}.
For the reconstruction bound, conditional expectation is the
$L^2$ projection onto $S_\ell$, so
\begin{equation}
(\delta_\ell^{\rm sat})^2
=\mathbb E\!\left[\operatorname{Var}
  (r^*(x)\mid z_{1:\ell})\right].
\end{equation}
Within a prefix cell $c$, let $x,x'$ be independent draws from
$P(\cdot\mid c)$. Lipschitz continuity yields
\begin{align}
\operatorname{Var}(r^*(x)\mid c)
&=\tfrac12\mathbb E[(r^*(x)-r^*(x'))^2\mid c] \\
&\leq \tfrac{L^2}{2}\mathbb E[\|x-x'\|^2\mid c]
 =L^2\operatorname{Var}(x\mid c).
\end{align}
The conditional mean minimizes squared reconstruction error,
whereas $\hat x_\ell$ is merely constant on each prefix cell.
Hence
$\operatorname{Var}(x\mid c)
\leq\mathbb E[\|x-\hat x_\ell\|^2\mid c]$.
Averaging over cells proves the second inequality in
\eqref{eq:depth-decomposition}.

Finally, put $a^*=P_{A_\infty}r^*$ and
$r_\perp=r^*-a^*$. Then
$\|r_\perp\|=\delta_\infty$. For any $a_\ell\in A_\ell$,
projection contraction gives
\begin{align}
\operatorname{dist}(P_{S_\ell}r^*,A_\ell)
&\leq\|P_{S_\ell}r^*-a_\ell\| \\
&\leq\|P_{S_\ell}(a^*-a_\ell)\|
     +\|P_{S_\ell}r_\perp\| \\
&\leq\|a^*-a_\ell\|+\delta_\infty.
\end{align}
Taking the infimum over $a_\ell$ proves the interaction-tail
bound. Combining it with the first two claims gives
\eqref{eq:depth-excess}.
\end{proof}

\begin{proof}[Proof of Corollary~\ref{cor:depth-rate}]
Theorem~\ref{thm:depth-rate} and
$\sqrt{u^2+v^2}\leq u+v$ for $u,v\geq0$ imply
\begin{equation}
\delta_\ell
\leq L\sqrt{\mathrm{MSE}_\ell}
   +\delta_\infty+\tau_\ell.
\end{equation}
Substitution of the two geometric assumptions gives the stated
bound with $\rho=\max\{\sqrt{\rho_x},\rho_\tau\}$.
\end{proof}

\section{Promotion Proofs}\label{app:promotion_proofs}

\noindent\textbf{Theorem~\ref{thm:shadow}} (Shadow false-promotion control, restated)\textbf{.}
\textit{The probability that the shadow is promoted despite having no lower
expected squared prediction error than the active model,
under the active policy's data distribution over the tested window,
is $\le \alpha$ over the full horizon.}

\noindent Shadow promotion differs from
CORRAL~\cite{agarwal2017corralling} in a structurally
important way: CORRAL runs $M$ learners simultaneously
and pays $\tilde{O}(\sqrt{M})$ regret overhead from
model selection; our shadow trains passively and incurs
\emph{zero} regret overhead during evaluation.
The regret during each phase is bounded by
Theorem~\ref{thm:asymptotic} at the \emph{active}
depth, with no model-selection penalty.
We test the paired squared-error differences via
Freedman's inequality~\cite{freedman1975tail}
(no i.i.d.\ assumption needed).

\begin{remark}[Freedman vs.\ t-test]\label{rem:freedman_ttest}
The formal guarantee (Theorem~\ref{thm:shadow}) requires
the Freedman-based test, which makes no i.i.d.\
assumption on the paired differences~$D_t$.
In our experiments, we also evaluate a paired $t$-test
variant as a practical heuristic; because the $D_t$
within each trailing window are not i.i.d.\ (the policy
evolves), the $t$-test $p$-values are not formally
calibrated, though empirically both tests yield similar
promotion decisions.
\end{remark}

\begin{remark}[Post-promotion covariate shift]\label{rem:covariate}
The formal guarantee is false-promotion control (Type~I error
$\le \alpha$). Post-promotion, the policy changes and arm
frequencies shift. For all three variants this shift is structurally bounded:
centroid assignments depend only on $x_t$ and the frozen codebook,
and TS ensures all cells are visited infinitely often
(Lemma~\ref{lem:explore_converge}(a)).
\end{remark}

\begin{proof}[Proof of Theorem~\ref{thm:shadow} (Shadow false-promotion control)]

\noindent\emph{Martingale formulation.}
Since the shadow updates only its own statistics
(line~\ref{alg:shadow}) and does not appear in the scoring
(line~\ref{alg:score}) or arm selection (line~\ref{alg:select}),
the action sequence is identical to Algorithm~\ref{alg:rqbandit}
at fixed depth.
To decide whether the shadow model genuinely outperforms the active
model, we construct a statistical test. At each round $t$, define
the paired squared-error difference between the active and shadow
predictions. Let $\hat{y}_{\mathrm{act}}(x_t, a_t)$ and
$\hat{y}_{\mathrm{shd}}(x_t, a_t)$ denote the posterior mean
predictions of the active and shadow models, respectively. Define:
\begin{equation}\label{eq:paired_diff}
    D_t = \left(r_t - \hat{y}_{\mathrm{act}}(x_t, a_t)\right)^2
        - \left(r_t - \hat{y}_{\mathrm{shd}}(x_t, a_t)\right)^2.
\end{equation}
When $D_t > 0$, the shadow's prediction was closer to the observed
reward at round $t$.

Since the shadow does not influence $a_t$, the conditional
expectation $\mathbb{E}[D_t \mid \mathcal{F}_{t-1}]$ is a
deterministic function of the models' states. Define the centered
differences:
\begin{equation}
    Z_t = D_t - \mathbb{E}[D_t \mid \mathcal{F}_{t-1}].
\end{equation}
By construction, $\mathbb{E}[Z_t \mid \mathcal{F}_{t-1}] = 0$,
so $(Z_t)_{t \ge 1}$ is a martingale difference sequence adapted to
$(\mathcal{F}_t)$. Rewards lie in $\{0,1\}$.
Since $\sigma$ maps to $[0,1]$, the posterior mean
predictions satisfy
$\hat{y}_{\mathrm{act}}, \hat{y}_{\mathrm{shd}} \in [0,1]$
for all three variants.
Each squared error is therefore in $[0, 1]$,
so $|D_t| \le 1 \eqqcolon R_D$ and:
\begin{equation}
    |Z_t| \le |D_t| + |\mathbb{E}[D_t \mid \mathcal{F}_{t-1}]|
    \le 2R_D \quad \text{a.s.}
\end{equation}

\smallskip
\noindent\emph{Freedman-based promotion test.}
Fix a check time $\tau$ and a trailing window of $W$ observations.
The cumulative difference decomposes into signal and noise:
\begin{equation}\label{eq:signal_noise}
    \bar{D}_W = \sum_{t=\tau-W+1}^{\tau} D_t
    = \underbrace{\sum_{t=\tau-W+1}^{\tau}
      \mathbb{E}[D_t \mid \mathcal{F}_{t-1}]}_{\text{signal: }
      \mu_W}
    + \underbrace{\sum_{t=\tau-W+1}^{\tau} Z_t}_{\text{noise: }
      M_W}.
\end{equation}
The signal $\mu_W$ is positive when the shadow is genuinely better
(in expected squared error) over the window. We need to distinguish
$\mu_W > 0$ from noise fluctuations.

\smallskip
\noindent
The noise $M_W = \sum_{t} Z_t$ is a sum of martingale
differences with $|Z_t| \le R \coloneqq 2R_D$.
Its predictable quadratic variation is:
\begin{equation}
    V_W = \sum_{t=\tau-W+1}^{\tau}
    \mathbb{E}\!\left[Z_t^2 \mid \mathcal{F}_{t-1}\right]
    \le R^2 W.
\end{equation}
By Freedman's inequality~\cite{freedman1975tail}:
\begin{equation}\label{eq:freedman}
    P\!\left(M_W \ge a \;\text{ and }\; V_W \le v\right)
    \le \exp\!\left(\frac{-a^2/2}{v + Ra/3}\right).
\end{equation}

\smallskip
\noindent To derive the threshold $a^*$,
we set $v = R^2 W$ (the worst-case variance bound) and require the
right-hand side of~\eqref{eq:freedman} to equal the per-test
significance level $\alpha_k$. This gives the equation:
\begin{equation}
    \exp\!\left(\frac{-a^2/2}{R^2 W + Ra/3}\right) = \alpha_k.
\end{equation}
Taking logarithms and writing $L = \log(1/\alpha_k)$:
\begin{equation}
    \frac{a^2}{2} = L\!\left(R^2 W + \frac{Ra}{3}\right)
    \quad \Longrightarrow \quad
    a^2 - \frac{2RL}{3}\,a - 2R^2 WL = 0.
\end{equation}
Applying the quadratic formula ($a > 0$):
\begin{align}
    a^*(\alpha_k, W)
    &= \frac{RL}{3}
       + \sqrt{\frac{R^2 L^2}{9} + 2R^2 WL}
    \nonumber \\
    &= \frac{R}{3}\log\frac{1}{\alpha_k}
       + R\sqrt{\frac{1}{9}\!\left(\log\frac{1}{\alpha_k}\right)^{\!2}
         + 2W\log\frac{1}{\alpha_k}}.
    \label{eq:threshold}
\end{align}
We promote if $\bar{D}_W > a^*(\alpha_k, W)$.

\smallskip
\noindent
Under the null hypothesis (shadow is not better, $\mu_W \le 0$),
the signal is non-positive, so from~\eqref{eq:signal_noise}:
\begin{equation}
    \bar{D}_W = \mu_W + M_W \le M_W.
\end{equation}
Therefore, a promotion requires $M_W > a^*$, which by
Freedman's inequality~\eqref{eq:freedman} occurs with probability
at most $\alpha_k$:
\begin{equation}
    P(\text{promote} \mid \mu_W \le 0)
    \le P(M_W > a^*)
    \le \alpha_k.
\end{equation}
Here, $a^*$ is the critical value that makes the noise term
unlikely to exceed the threshold, ensuring false promotions are rare.

\smallskip
\noindent\emph{Bonferroni correction.}
The doubling schedule checks at times $\{T_0, 2T_0, 4T_0, \ldots\}$,
producing at most $n_c$ check intervals:
\begin{equation}
    n_c = \lfloor\log_2(T / T_0)\rfloor + 1.
\end{equation}
At each check time, the algorithm only evaluates the immediate
shadow depth $k{+}1$ extending from the current active depth~$k$.
Therefore, it conducts exactly one hypothesis test per scheduled
check interval.
An additional test is only conducted if a promotion successfully
advances the active frontier, prompting the algorithm to immediately
evaluate the next depth.
Note that this immediate re-test uses the trailing window $W$,
which contains data collected under the previous depth's policy;
the conservatism of the Freedman-based test
(Theorem~\ref{thm:shadow}) provides robustness against
this distribution shift.
Because at most $\ell_{\max} - 1$ promotions can ever occur
(from depth~1 to depth~$\ell_{\max}$), the total number of
hypothesis tests over the entire horizon is bounded by:
\begin{equation}
    N \le n_c + (\ell_{\max} - 1).
\end{equation}
Setting $\alpha_k = \alpha / N$ for each individual test and
applying the union bound:
\begin{equation}\label{eq:bonferroni}
    P(\text{any false promotion})
    \le \sum_{k=1}^{N} \alpha_k
    = N \cdot \frac{\alpha}{N}
    = \alpha.
\end{equation}
By tracking only the frontier tests, the multiplicity penalty
is additive $\mathcal{O}(n_c + \ell_{\max})$ rather than
multiplicative $\mathcal{O}(n_c \cdot \ell_{\max})$.
The doubling schedule keeps $n_c = \mathcal{O}(\log T)$.

\smallskip
\noindent\emph{Combining.}
By~\eqref{eq:bonferroni}, with probability at least $1 - \alpha$,
\emph{every} promotion is a true improvement under the active
policy's data distribution:
\begin{equation}
    \mu_W > 0,
\end{equation}
meaning the shadow has strictly lower expected squared error on the
$(x_t, a_t)$ pairs generated by $\pi_{\mathrm{act}}$. This is a
\emph{pre-promotion} guarantee; post-promotion, the policy changes
and the data distribution shifts
(see Remark~\ref{rem:covariate}).

Regardless of whether any individual promotion is true or false,
the shadow never influences arm selection, so the regret during
each phase (at fixed depth $\ell_k$) is bounded by
Theorem~\ref{thm:asymptotic}.
If promotions occur at times
$T_1, T_2, \ldots$ with depths $\ell_1, \ell_2, \ldots$, the total
regret is the sum across phases:
\begin{equation}
    R(T) \le \sum_{k} R(T_k - T_{k-1},\, \ell_k),
\end{equation}
where each term is bounded by Theorem~\ref{thm:asymptotic} at the
corresponding depth. Since $\Delta_{\ell_{k+1}} \le \Delta_{\ell_k}$
(deeper models have lower bias), promotions can only reduce the
bias term in subsequent phases.
\end{proof}
\section{Depth Ablation and Exponential Decay}\label{app:depth_ablation}

\noindent This appendix examines empirical behavior relevant to
the conditional geometric rate in Corollary~\ref{cor:depth-rate}
(\S\ref{sec:analysis}) and
the depth results in \S\ref{sec:exp-depth}.

\subsection{Exponential Decay of Misspecification}\label{app:exp_decay}

\noindent\textbf{Conditional depth-rate prediction.}
Corollary~\ref{cor:depth-rate} gives geometric convergence to
an irreducible floor when both RQ reconstruction error and the
additive interaction tail contract geometrically.

\noindent To test whether the observed behavior is consistent
with this prediction, we compute the
regret rate $R(T)/T$ at the terminal round for each
(dataset, method, depth) combination using depths
$\ell = 1, \ldots, 5$.
Since $R(T)/T \to O(\Delta_\ell)$ as $T \to \infty$
(Theorem~\ref{thm:asymptotic}), the converged regret rate
provides a proxy for $\Delta_\ell$.

\begin{figure}[H]
\centering
\includegraphics[width=\textwidth]{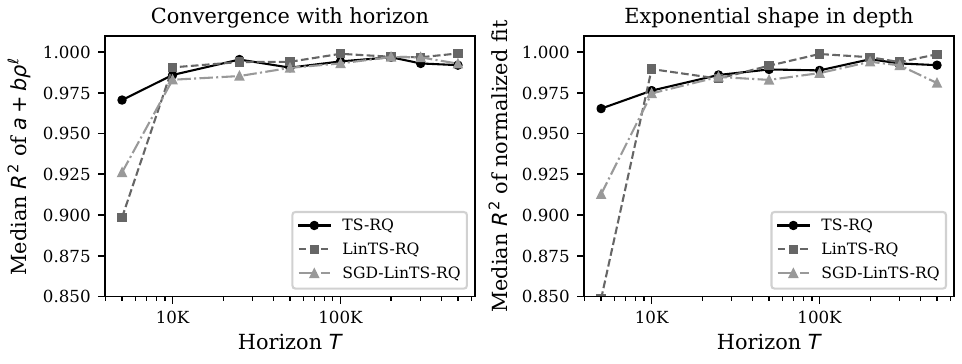}
\caption{Empirical behavior consistent with the conditional
geometric depth-rate law,
one column per method.
\textbf{Top:} Median $R^2$ of the exponential fit
$a + b\rho^\ell$ as a function of horizon~$T$;
fit quality improves with more data.
\textbf{Bottom:} Normalized regret decay
$(\Delta_\ell - \Delta_5)/(\Delta_1 - \Delta_5)$ at
terminal~$T$ for each dataset.}
\label{fig:delta_decay_summary}
\end{figure}

\smallskip
\noindent\textbf{Exponential fit.}
We fit $R(T)/T \approx a + b \rho^\ell$ to each
(dataset, method) pair via nonlinear least squares,
yielding 39 fits (13 datasets $\times$ 3 RQ methods,
30 seeds each).
The table below reports the median fit
quality.

\begin{center}
\begin{tabular}{@{}l c c c@{}}
\toprule
\textbf{Metric} & \textbf{Median} & \textbf{Mean} & $> 0.90$ \\
\midrule
$R^2$ of 3-parameter fit & 0.99 & 0.94 & 24/28 (86\%) \\
\bottomrule
\end{tabular}
\end{center}

\smallskip
\noindent\textbf{Leave-one-out cross-validation.}
For each (dataset, method), we hold out one depth, fit
on the remaining four, and predict the held-out regret rate.
This produces 135 predictions (some fits fail on 3 points).
The median prediction error is $0.3\%$; 73\% of predictions
fall within $1\%$ of the observed value
(Figure~\ref{fig:loo_cv}).

\begin{figure}[H]
\centering
\includegraphics[width=\textwidth]{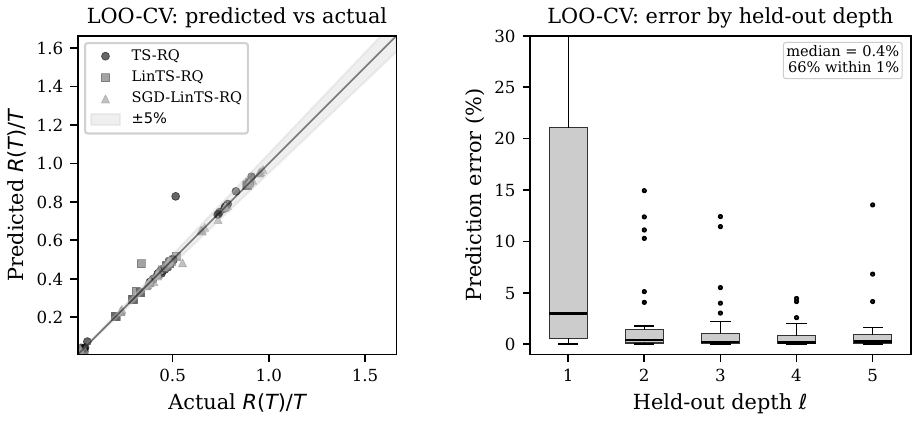}
\caption{Leave-one-out cross-validation of the exponential
decay model $\Delta_\ell \approx a + b\rho^\ell$.
\textbf{Left:} Predicted vs.\ actual regret rate $R(T)/T$
for all (dataset, method, held-out depth) triples; the gray
band shows $\pm 5\%$.
\textbf{Right:} Prediction error by held-out depth;
depths $2$--$5$ are predicted within $\sim\!1\%$, while
depth~$1$ (furthest from $\Delta_\infty$) shows larger
extrapolation error.}
\label{fig:loo_cv}
\end{figure}

\smallskip
\noindent\textbf{Exponential decay rate.}
Figure~\ref{fig:wcv_decay} fits
$a + b\rho^\ell$ to the normalized within-cell variance
$\mathrm{WCV}(\ell)/\mathrm{WCV}(0)$ for each dataset.
The median fitted decay rate is $\rho = 0.86$
(range $0.17$--$0.93$), with median $R^2 = 0.83$.

\begin{figure}[H]
\centering
\includegraphics[width=\textwidth]{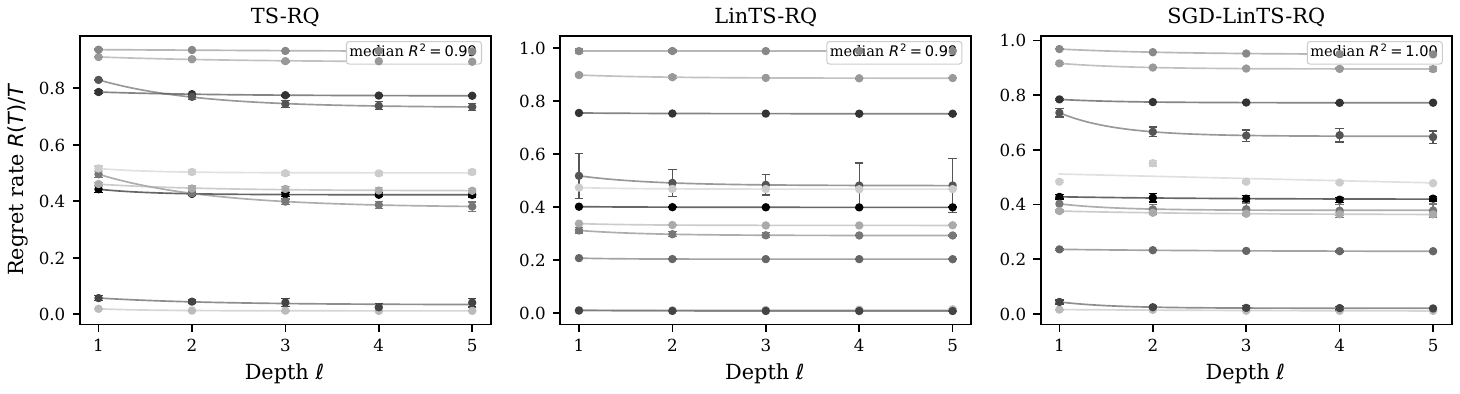}
\caption{Regret rate $R(T)/T$ vs.\ depth $\ell$ for each RQ method,
with fitted exponential decay curves $a + b\rho^\ell$ overlaid.
Each line is a dataset; error bars show $\pm 1$ std over 30 seeds.
Median $R^2$ is reported per method.}
\label{fig:delta_decay_by_method}
\end{figure}

\subsection{Within-Cell Reward Variance}\label{app:wcv}

Figure~\ref{fig:wcv_decay} measures within-cell reward
heterogeneity directly.
For each dataset we train an RQ codebook ($b{=}16$,
depth~5) and compute the weighted average within-cell
reward variance
$\mathrm{WCV}(\ell)
= \sum_c (n_c / N)\,\mathrm{Var}(r \mid \text{cell}_\ell{=}c)$,
where cells at depth~$\ell$ are defined by the first~$\ell$
centroid assignments.
On a log scale, all 13 datasets decay roughly linearly,
consistent with geometric shrinkage over the measured depths,
with median per-level
ratio $\rho = 0.83$ (range $0.24$--$0.93$).

\begin{figure}[H]
\centering
\includegraphics[width=0.75\textwidth]{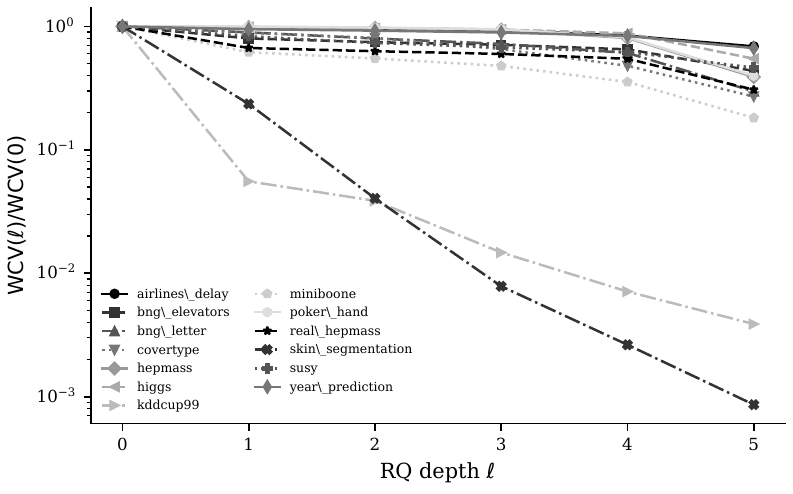}
\caption{Normalized within-cell reward variance
$\mathrm{WCV}(\ell)/\mathrm{WCV}(0)$ vs.\ RQ depth~$\ell$
for all 13 datasets (log scale).
Every dataset decays monotonically; the near-linear
trajectories are consistent with geometric decay with
median $\rho = 0.83$.}
\label{fig:wcv_decay}
\end{figure}

\subsection{Fixed-Depth vs.\ Shadow Promotion}\label{app:fixed_vs_shadow}

We compare TS-RQ at fixed depths $\ell{=}1$ and $\ell{=}2$
against the adaptive shadow-promoted variant (TS-RQ) on two
representative datasets.

\begin{center}
\begin{tabular}{@{}l r r r@{}}
\toprule
& \textbf{$\ell{=}1$} & \textbf{$\ell{=}2$}
& \textbf{Adaptive} \\
\midrule
covertype  & 48.3\% & 42.2\% & 38.8\% (depth 4) \\
higgs      & 45.4\% & 43.6\% & 43.7\% (depth 3) \\
\bottomrule
\end{tabular}
\end{center}

On covertype, shadow promotion discovers depth~4 and gains
$3.4$ percentage points over fixed $\ell{=}2$, confirming that
adaptive depth selection adds value beyond simply committing to
$\ell{=}2$.
On higgs, the benefit of depth beyond~2 is marginal, and the
adaptive method correctly settles at depth~3 with regret matching
$\ell{=}2$.

\section{Analytical Variant: Finite-Time Bound}\label{app:analytical_variants_proof}

\noindent This appendix analyzes an analytical variant of
Algorithm~\ref{alg:rqbandit} that maintains the full
$D \times D$ joint posterior covariance rather than the
block-diagonal approximation used in practice.
The variant admits an explicit finite-time regret bound
(Proposition~\ref{prop:regret} below), complementing the
asymptotic guarantee for Algorithm~\ref{alg:rqbandit}
(Theorem~\ref{thm:asymptotic}).
The finite-time bound is a standard application of
LinTS analysis~\cite{agrawal2013thompson} to our feature
map; the contribution is not the proof technique but the
\emph{architectural insight} that the RQ additive
structure ensures $D = b\ell p$ while achieving the
approximation power of $b^\ell$ cells---an exponential
compression not available to flat partitions.

\smallskip
\noindent\textbf{Analytical variant details.}
The analytical variant (Proposition~\ref{prop:regret})
concatenates all per-level features into a block vector
$\Phi_{\mathcal{M}}(x) \in \mathbb{R}^{D}$ with
$D = b\ell\, p(\mathcal{M})$:
\begin{equation}\label{eq:feature_map}
    \Phi_{\mathcal{M}}(x) =
    \begin{pmatrix}
        \phi_1(x) \\ \phi_2(x) \\ \vdots \\ \phi_\ell(x)
    \end{pmatrix}
    \!\in \mathbb{R}^{D},
    \qquad
    \Sigma_a =
    \begin{pmatrix}
        \Sigma_{11} & \cdots & \Sigma_{1\ell} \\
        \vdots & \ddots & \vdots \\
        \Sigma_{\ell 1} & \cdots & \Sigma_{\ell\ell}
    \end{pmatrix}
    \!\in \mathbb{R}^{D \times D},
\end{equation}
where each $\phi_i(x) \in \mathbb{R}^{bp}$ encodes the
level-$i$ assignment
(one-hot for TS-RQ, context-at-centroid for LinTS-RQ)
and $\Sigma_{ij} \in \mathbb{R}^{bp \times bp}$ captures
cross-level correlations.\label{eq:joint_cov}
A single parameter vector
$\theta_a \in \mathbb{R}^D$ recovers the
\emph{unclipped} additive model:
$r(x,a) = \theta_a^\top \Phi_{\mathcal{M}}(x)$.

The analytical variant runs LinTS on $\Phi_{\mathcal{M}}(x)$
with the full $D \times D$ posterior covariance
$\Sigma_a$.
Define the \emph{variant misspecification}
$\Delta(\Phi) = \mathbb{E}_{x}[\max_a
|r^*(x,a) - \theta_a^{*\top}\Phi(x)|]$,
where $\theta_a^*$ is the best linear predictor for arm~$a$
on $\Phi$.
Since $\theta_a^\top \Phi(x) = \sum_i f_i(k_i(x), a)$
recovers any unclipped additive model,
$\Delta(\Phi) \le \Delta_\ell^{\mathrm{uncl}}$
where $\Delta_\ell^{\mathrm{uncl}}$ is the error of the
stagewise limit without clipping.
In practice, clipping ($\sigma$) can only reduce
approximation error for bounded rewards, so
$\Delta_\ell(\mathcal{M}) \le
\Delta_\ell^{\mathrm{uncl}}$, giving
$\Delta(\Phi)$ and $\Delta_\ell(\mathcal{M})$ the same
asymptotic scaling in depth.

\smallskip
\noindent\textbf{Complexity comparison.}
Table~\ref{tab:complexity} summarizes the exploration and memory
costs of the analytical variant versus Algorithm~\ref{alg:rqbandit}.

\begin{table}[H]
\centering
\caption{Exploration and memory scaling by base learner.
``Joint'' columns refer to the analytical variant;
``Actual'' is Algorithm~\ref{alg:rqbandit}.}
\label{tab:complexity}
\begin{tabular}{@{}l c c c c@{}}
\toprule
$\mathcal{M}$ & $p(\mathcal{M})$
  & Exploration & Joint memory & Actual memory \\
\midrule
TS & $1$ & $O(b\ell\sqrt{KT\log(T/\delta)})$
  & $O(b^2\ell^2 K)$ & $O(b\ell K)$ \\
LinTS & $d$ & $O(b\ell d\sqrt{KT\log(T/\delta)})$
  & $O(b^2\ell^2 d^2 K)$ & $O(b\ell d^2 K)$ \\
SGD-LinTS & $d$ & $O(b\ell d\sqrt{KT\log(T/\delta)})$
  & $O(b^2\ell^2 d^2 K)$ & $O(b\ell d K)$ \\
\bottomrule
\end{tabular}
\end{table}

The analytical variant runs $K$ independent LinTS
instances~\cite{agrawal2013thompson}, one per arm, on the
$D$-dimensional feature vector
$\Phi_{\mathcal{M}}(x) \in \mathbb{R}^D$ with
$D = b\ell\, p(\mathcal{M})$~\eqref{eq:feature_map}.
Each arm~$a$ maintains a ridge precision matrix
\begin{equation}
    V_{a,t} = \lambda I +
    \sum_{s \le t:\, a_s=a} \Phi(x_s)\Phi(x_s)^\top
    \in \mathbb{R}^{D \times D}.
\end{equation}
Let $r^*(x,a) = \mathbb{E}[r_t \mid x_t{=}x,\, a_t{=}a]$
denote the true expected reward, and let
$\theta_a^* = \arg\min_\theta
\mathbb{E}_{x}[(r^*(x,a) - \theta^\top \Phi(x))^2]$
be the best linear predictor for arm~$a$.
Define the misspecification function
$\epsilon(x, a) = r^*(x, a) - \theta_a^{*\top}\Phi(x)$
and the observation noise
$\eta_t = r_t - r^*(x_t, a_t)$.
Note that $\Delta(\Phi) = \mathbb{E}_{x}[\max_a |\epsilon(x, a)|]$.
The observed reward decomposes as
\begin{equation}
    r_t = \theta_{a_t}^{*\top}\Phi(x_t)
    + \underbrace{\epsilon(x_t, a_t)}_{\text{bias}}
    + \underbrace{\eta_t}_{\text{noise}}.
\end{equation}
At time~$t$, the ridge posterior mean estimates $\theta_a^*$ by
\begin{align}
    \hat{\theta}_{a,t}
    &= V_{a,t}^{-1} \sum_{s:a_s=a} \Phi(x_s)\, r_s
    \nonumber \\
    &= \underbrace{V_{a,t}^{-1} \sum_{s:a_s=a}
       \Phi(x_s)(\theta_a^{*\top}\Phi(x_s) + \eta_s)}_{
       \hat{\theta}_{a,t}^{\mathrm{perf}}}
    \;+\; \underbrace{V_{a,t}^{-1} \sum_{s:a_s=a}
       \Phi(x_s)\,\epsilon(x_s, a_s)}_{
       \hat{\theta}_{a,t}^{\mathrm{bias}}}
    ,
\end{align}
where the first term is the solution under perfect specification,
and the second is the bias due to misspecification.

\begin{lemma}[Bias regret]\label{lem:misspec}
For i.i.d.\ contexts,
$\mathbb{E}_{x}[\max_a |\epsilon(x, a)|]
= \Delta(\Phi)$.
\end{lemma}

\begin{proof}
Immediate from the definitions of $\epsilon$ and
$\Delta(\Phi)$.
\end{proof}

\begin{lemma}[Exploration regret under realizability]\label{lem:exploration}
For the analytical variant under realizability
($\Delta(\Phi) = 0$, i.e.\ $\epsilon \equiv 0$),
with bounded features $\|\Phi(x)\| \le C$, rewards
in $[0,1]$, and $D$-dimensional parameter vectors,
the cumulative regret satisfies,
with probability $\ge 1 - \delta$,
\begin{equation}
    R(T, \ell)
    = O\!\left(D\sqrt{KT\log(T/\delta)}\right).
\end{equation}
\end{lemma}

\begin{proof}
Under realizability, $r^*(x,a) = \theta_a^{*\top}\Phi(x)$
for all $(x,a)$, so the reward noise
$\eta_t = r_t - \theta_{a_t}^{*\top}\Phi(x_t)$ is
zero-mean conditional on $(x_t, a_t)$.
This is exactly the setting
of~\cite{agrawal2013thompson}, Theorem~2 (stated for Gaussian noise;
the extension to sub-Gaussian noise, which covers our bounded binary
rewards, is standard~\cite{lattimore2020bandit}), giving
per-arm regret $O(D\sqrt{T_a \log(T/\delta)})$.
Summing over arms ($\sum_a T_a = T$):
\begin{align}
    \sum_{a=1}^{K} O\!\left(D \sqrt{T_a \log(T/\delta)}\right)
    &= O\!\left(D\sqrt{\log(T/\delta)}\right)
         \sum_{a=1}^{K} \sqrt{T_a}
    \nonumber \\
    &\stackrel{\text{Jensen}}{\le}
         O\!\left(D\sqrt{\log(T/\delta)}\right)
         \sqrt{K \sum_{a=1}^{K} T_a}
    \nonumber \\
    &\stackrel{\sum T_a = T}{=}
         O\!\left(D \sqrt{KT\log(T/\delta)}\right).
         \qedhere
\end{align}
\end{proof}

\begin{remark}[Variant regret (heuristic bound)]\label{prop:regret}
For the analytical variant with bounded features
($\|\Phi(x)\| \le C$) and rewards in $[0,1]$,
under the heuristic decomposition below,
with probability $\ge 1 - \delta$,
\begin{equation}\label{eq:regret}
    \frac{R(T, \ell)}{T} \le
    \underbrace{O\!\left(\Delta(\Phi)
    \right)\vphantom{\sqrt{T}}}_{\text{approx.\ bias}}
    + \underbrace{O\!\left(\frac{D
    \sqrt{K \log(T/\delta)}}{\sqrt{T}}\right)}_{\text{exploration}}
    \;\xrightarrow{T \to \infty}\;
    O\!\left(\Delta(\Phi)\right).
\end{equation}
\end{remark}

\begin{proof}[Argument for Remark~\ref{prop:regret}]

Let $a_t^* = \arg\max_a r^*(x_t, a)$.
The per-round regret decomposes as
\begin{align}
    r^*(x_t, a_t^*) - r^*(x_t, a_t)
    &= \underbrace{\bigl[\theta_{a_t^*}^{*\top}\Phi(x_t)
       - \theta_{a_t}^{*\top}\Phi(x_t)\bigr]}_{\text{linear suboptimality}}
    \nonumber \\
    &\quad + \underbrace{\bigl[\epsilon(x_t, a_t^*)
       - \epsilon(x_t, a_t)\bigr]}_{\text{misspecification}}.
    \label{eq:regret_decomp}
\end{align}
The misspecification bracket satisfies
$|\epsilon(x_t, a_t^*) - \epsilon(x_t, a_t)|
\le 2\max_a |\epsilon(x_t, a)|$.
Since contexts are i.i.d., Hoeffding's inequality gives
$\sum_{t=1}^T \max_a |\epsilon(x_t, a)|
\le T \cdot \Delta(\Phi)
+ O(\sqrt{T \log(1/\delta)})$
with probability $\ge 1 - \delta$; the second term is
absorbed into the exploration term below.

For the linear suboptimality, we apply
Lemma~\ref{lem:exploration}.
Strictly, Lemma~\ref{lem:exploration} is stated under
realizability ($\epsilon \equiv 0$); in the misspecified
case, the systematic bias~$\epsilon$ shifts the posterior
mean, which can affect the algorithm's exploration
behavior.
For UCB-style algorithms, this additive decomposition
is standard~\cite{lattimore2020bandit}; for Thompson
Sampling, the shifted posterior may alter the sampling
distribution in non-additive ways.
Following~\cite{lattimore2020learning}, we treat the
two terms as independent, yielding the bound below as
an upper bound on the well-specified component of regret.
The asymptotic guarantee
(Theorem~\ref{thm:asymptotic}) is established
independently via the property-based proof and does not
rely on this decomposition.
With $D = b\ell\, p(\mathcal{M})$:
\begin{equation}\label{eq:total_regret_proof}
    R(T, \ell) \le
    \underbrace{O\!\left(b\ell\, p(\mathcal{M})
    \sqrt{KT\log(T/\delta)}\right)}_{\text{exploration}}
    + \underbrace{O\!\left(\Delta(\Phi)
    \cdot T\right)}_{\text{misspecification bias}}.
    \qedhere
\end{equation}
\end{proof}

\smallskip
\noindent\textbf{Variance deflation.}
The block-diagonal approximation used by
Algorithm~\ref{alg:rqbandit} inverts each level's
precision block independently.
The following lemma shows that this yields a posterior
that is \emph{tighter} than the exact marginal of the
joint posterior above.
For UCB-style algorithms, tighter confidence sets imply
lower regret; for Thompson Sampling, the relationship
between posterior width and regret is not monotone, so
the finite-time bound above applies to the analytical
variant only.
The asymptotic guarantee for Algorithm~\ref{alg:rqbandit}
(Theorem~\ref{thm:asymptotic}) is established
independently via the property-based proof in
Appendix~\ref{app:fixed_depth_proofs}.

\begin{lemma}[Marginal variance deflation]
\label{lem:var_deflation}
Let $\Sigma \succ 0$ be the joint posterior covariance
over all $\ell$ levels~\eqref{eq:feature_map},
partitioned into $bp \times bp$ blocks.
For every level $i \in \{1, \ldots, \ell\}$,
$([\Sigma^{-1}]_{ii})^{-1} \preceq [\Sigma]_{ii}$.
\end{lemma}

\begin{proof}
Partition $\Sigma^{-1} \succ 0$ into level~$i$ versus the
remaining $\ell - 1$ levels:
\[
    \Sigma^{-1} =
    \begin{pmatrix}
        A & B \\ B^\top & C
    \end{pmatrix},
\]
where $A = [\Sigma^{-1}]_{ii} \in \mathbb{R}^{bp \times bp}$
and $C \in \mathbb{R}^{(\ell{-}1)bp \times (\ell{-}1)bp}$.
Since $\Sigma^{-1} \succ 0$, both $A$ and $C$ are positive
definite.
By the block inversion formula,
\begin{equation}\label{eq:schur_marginal}
    [\Sigma]_{ii} = (A - B C^{-1} B^\top)^{-1}.
\end{equation}
Since $C \succ 0$, the Schur complement satisfies
$B C^{-1} B^\top \succeq 0$, giving
$A - B C^{-1} B^\top \preceq A$.
Both sides are positive definite (the left-hand side by
$\Sigma^{-1} \succ 0$), so inverting reverses the Loewner
order~\cite[Theorem~7.7.6]{horn2012matrix}:
\[
    ([\Sigma^{-1}]_{ii})^{-1} = A^{-1}
    \preceq (A - B C^{-1} B^\top)^{-1}
    = [\Sigma]_{ii}.
    \qedhere
\]
\end{proof}

\section{Dataset Details}\label{app:datasets}

\noindent\textbf{Codebook training.}
We train the RQ codebook via scikit-learn's
\texttt{MiniBatchKMeans} ($k{=}16$, \texttt{init='k-means++'},
\texttt{n\_init=3}, \texttt{max\_iter=100}, \texttt{batch\_size=1024},
\texttt{random\_state=42}) applied sequentially: we train level~$i$
on the context residuals $x^{(i-1)}$ from the $n_{\mathrm{rq}}{=}50\text{K}$
holdout contexts.
All features are standardized to zero mean and unit variance before
codebook training.

\smallskip
\noindent\textbf{Preprocessing.}
We label-encode categorical features.
We apply no feature selection, imputation, or augmentation.

\begin{table}[H]
\centering
\caption{Datasets used in the evaluation. $N$ = total samples,
$d$ = feature dimension, $K$ = number of arms (classes),
$p_{\max}$ = majority class frequency,
$\text{eff.\ rank}$ = effective rank of the feature covariance.
All datasets are publicly available from OpenML and
UCI~\cite{dua2019uci}; codebook trained on
$n_{\mathrm{rq}}{=}50\text{K}$ held-out samples.
Individual sources:
HIGGS and SUSY~\cite{baldi2014searching},
HEPMASS~\cite{baldi2016parameterized},
covertype~\cite{blackard1999comparative},
MiniBooNE~\cite{roe2005boosted},
year\_prediction~\cite{bertin2011million}.}
\label{tab:datasets}
\begin{tabular}{@{}l r r r r r@{}}
\toprule
\textbf{Dataset} & $N$ & $d$ & $K$ & $p_{\max}$ & eff.\ rank \\
\midrule
airlines\_delay   & 500K  & 7   & 2   & 0.555 & 4.8 \\
bng\_elevators    & 1M    & 18  & 10  & 0.183 & 3.7 \\
bng\_letter       & 1M    & 16  & 26  & 0.041 & 4.6 \\
covertype         & 581K  & 54  & 7   & 0.488 & 3.9 \\
beer\_reviews      & 500K  & 12  & 104 & 0.074 & 3.2 \\
hepmass           & 5M    & 27  & 2   & 0.500 & 5.7 \\
higgs             & 11M   & 28  & 2   & 0.530 & 6.7 \\
kddcup99          & 494K  & 41  & 17  & 0.568 & 3.9 \\
miniboone         & 130K  & 50  & 2   & 0.719 & 1.1 \\
poker\_hand       & 500K  & 10  & 9   & 0.502 & 9.7 \\
skin\_seg.        & 245K  & 3   & 2   & 0.793 & 1.3 \\
susy              & 5M    & 18  & 2   & 0.542 & 3.7 \\
year\_prediction  & 500K  & 90  & 84  & 0.076 & 8.7 \\
\bottomrule
\end{tabular}
\end{table}

\section{Full Per-Dataset Results}\label{app:full_results}

Tables~\ref{tab:full_lintsrq}--\ref{tab:full_sgdrq} show per-dataset
results for all three method families.

\begin{table}[h]
\centering
\caption{LinTS-RQ vs LinTS: per-dataset regret rates (\%) and gap with 95\% CIs.}
\label{tab:full_lintsrq}
\begin{tabular}{@{}l r r r r r r@{}}
\toprule
\textbf{Dataset} & $d$ & $K$ & \textbf{Base (\%)} & \textbf{RQ (\%)} & \textbf{Gap (\%)} & $T$ \\
\midrule
airlines & 7 & 2 & 41.3 & 39.8 & $-3.6 \pm 0.1$ & 450K \\
beer\_rev. & 12 & 104 & 85.6 & 83.0 & $-2.9 \pm 0.5$ & 450K \\
bng\_elev. & 18 & 10 & 75.1 & 75.0 & $-0.2 \pm 0.0$ & 950K \\
bng\_letter & 16 & 26 & 52.4 & 46.6 & $-11.0 \pm 0.2$ & 950K \\
covertype & 54 & 7 & 28.1 & 23.9 & $-15.0 \pm 0.2$ & 530K \\
hepmass & 27 & 2 & 18.6 & 18.4 & $-1.1 \pm 0.1$ & 830K \\
higgs & 28 & 2 & 35.5 & 32.6 & $-8.0 \pm 0.1$ & 6.7M \\
kddcup99 & 41 & 17 & 0.2 & 0.2 & $-16.0 \pm 17.8$ & 440K \\
miniboone & 50 & 2 & 11.0 & 9.9 & $-10.1 \pm 0.7$ & 80K \\
poker & 10 & 9 & 49.4 & 46.9 & $-5.0 \pm 0.3$ & 445K \\
skin\_seg. & 3 & 2 & 8.3 & 0.3 & $-96.1 \pm 0.2$ & 195K \\
susy & 18 & 2 & 21.6 & 20.3 & $-6.1 \pm 0.1$ & 3.5M \\
year\_pred. & 90 & 84 & 93.6 & 97.8 & $+4.5 \pm 0.1$ & 445K \\
\midrule
\multicolumn{7}{l}{\textbf{11W / 1L / 1T}} \\
\bottomrule
\end{tabular}
\end{table}

\begin{table}[h]
\centering
\caption{TS-RQ vs TS: per-dataset regret rates (\%) and gap with 95\% CIs.}
\label{tab:full_tsrq}
\begin{tabular}{@{}l r r r r r r@{}}
\toprule
\textbf{Dataset} & $d$ & $K$ & \textbf{Base (\%)} & \textbf{RQ (\%)} & \textbf{Gap (\%)} & $T$ \\
\midrule
airlines & 7 & 2 & 45.6 & 42.8 & $-6.2 \pm 2.5$ & 450K \\
beer\_rev. & 12 & 104 & 92.7 & 89.8 & $-3.2 \pm 0.3$ & 450K \\
bng\_elev. & 18 & 10 & 81.7 & 77.9 & $-4.6 \pm 0.2$ & 950K \\
bng\_letter & 16 & 26 & 96.0 & 71.4 & $-25.6 \pm 0.5$ & 950K \\
covertype & 54 & 7 & 51.3 & 38.8 & $-24.4 \pm 0.7$ & 530K \\
hepmass & 27 & 2 & 50.0 & 21.7 & $-56.6 \pm 0.2$ & 830K \\
higgs & 28 & 2 & 48.0 & 43.7 & $-9.0 \pm 1.6$ & 6.7M \\
kddcup99 & 41 & 17 & 42.7 & 1.1 & $-97.4 \pm 0.1$ & 440K \\
miniboone & 50 & 2 & 28.3 & 14.7 & $-48.1 \pm 0.8$ & 80K \\
poker & 10 & 9 & 49.8 & 50.4 & $+1.1 \pm 0.6$ & 445K \\
skin\_seg. & 3 & 2 & 20.8 & 2.6 & $-87.4 \pm 2.4$ & 195K \\
susy & 18 & 2 & 46.6 & 24.8 & $-46.8 \pm 1.1$ & 3.5M \\
year\_pred. & 90 & 84 & 92.8 & 93.2 & $+0.4 \pm 0.2$ & 445K \\
\midrule
\multicolumn{7}{l}{\textbf{11W / 2L / 0T}} \\
\bottomrule
\end{tabular}
\end{table}

\begin{table}[h]
\centering
\caption{SGD-RQ vs SGD-LinTS: per-dataset regret rates (\%) and gap with 95\% CIs.}
\label{tab:full_sgdrq}
\begin{tabular}{@{}l r r r r r r@{}}
\toprule
\textbf{Dataset} & $d$ & $K$ & \textbf{Base (\%)} & \textbf{RQ (\%)} & \textbf{Gap (\%)} & $T$ \\
\midrule
airlines & 7 & 2 & 46.2 & 42.4 & $-8.2 \pm 1.1$ & 450K \\
beer\_rev. & 12 & 104 & 91.7 & 89.2 & $-2.8 \pm 0.2$ & 450K \\
bng\_elev. & 18 & 10 & 80.0 & 77.6 & $-3.0 \pm 0.6$ & 950K \\
bng\_letter & 16 & 26 & 81.6 & 64.3 & $-21.2 \pm 0.7$ & 950K \\
covertype & 54 & 7 & 48.9 & 32.0 & $-34.6 \pm 1.4$ & 530K \\
hepmass & 27 & 2 & 20.5 & 19.1 & $-7.1 \pm 0.2$ & 830K \\
higgs & 28 & 2 & 49.0 & 36.2 & $-26.2 \pm 2.3$ & 6.7M \\
kddcup99 & 41 & 17 & 20.7 & 1.2 & $-94.3 \pm 2.4$ & 440K \\
miniboone & 50 & 2 & 32.9 & 14.1 & $-57.3 \pm 6.1$ & 80K \\
poker & 10 & 9 & 53.6 & 49.4 & $-7.8 \pm 0.8$ & 445K \\
skin\_seg. & 3 & 2 & 42.0 & 2.5 & $-94.2 \pm 1.1$ & 195K \\
susy & 18 & 2 & 29.2 & 23.0 & $-21.2 \pm 0.6$ & 3.5M \\
year\_pred. & 90 & 84 & 93.3 & 94.0 & $+0.7 \pm 0.1$ & 445K \\
\midrule
\multicolumn{7}{l}{\textbf{12W / 1L / 0T}} \\
\bottomrule
\end{tabular}
\end{table}

\section{Per-Dataset Regret Curves}\label{app:curves}

\begin{figure}[H]
\centering
\includegraphics[width=\textwidth]{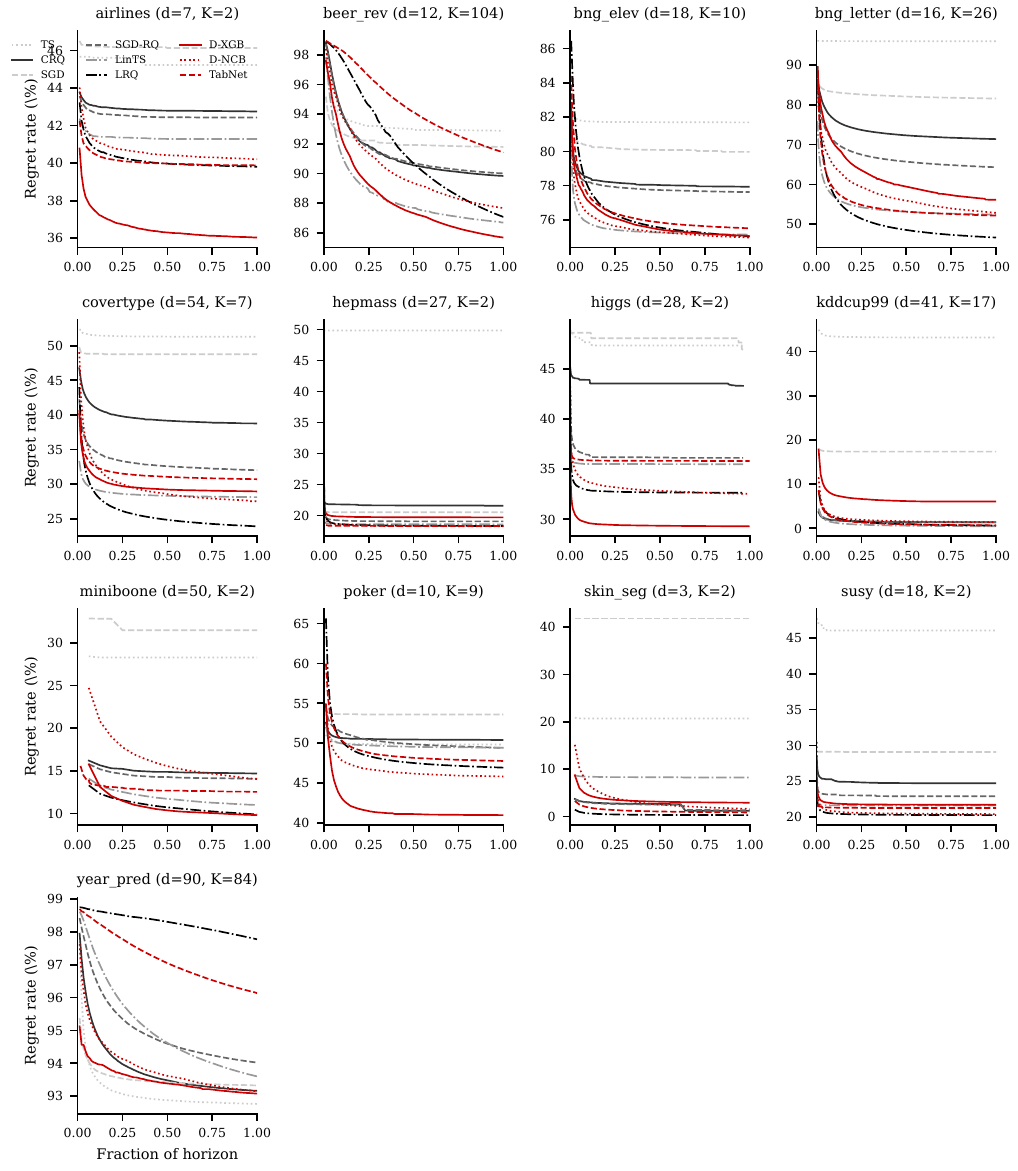}
\caption{Regret rate vs.\ $T$ for all 9 methods on each dataset.
Shaded bands show $\pm 1$ std across seeds.}
\label{fig:regret_curves_all}
\end{figure}

\section{Regret Scaling Analysis}\label{app:regret_scaling}

To empirically validate consistency with the regret
bound~\eqref{eq:regret}, we fit two models to the TS-RQ
cumulative regret curves across all datasets.

\smallskip
\noindent\textbf{Power-law fit.}
We fit $R(T) = a \cdot T^\alpha$ via log-log regression.
If regret were dominated by the $O(\sqrt{T\log T})$ exploration
term, we would expect $\alpha \approx 0.5$. If dominated by the
$O(\Delta_\ell T)$ bias term, we expect $\alpha \approx 1.0$.
Figure~\ref{fig:regret_scaling_full} (left) shows that all 13 datasets
yield $\alpha \in [0.87, 1.00]$, indicating that regret is
overwhelmingly dominated by the linear bias term at the horizons
in our experiments ($T$ up to $11$M).

\smallskip
\noindent\textbf{Decomposition fit.}
We also fit $R(T) = a\sqrt{T} + c \cdot T$ directly, separating
exploration and bias contributions.
The bias coefficient $c$ estimates the asymptotic regret rate
$R(T)/T \to O(\Delta_\ell(\mathcal{M}))$.
Figure~\ref{fig:regret_scaling_full} (right) shows the regret rate
converging to a constant on selected datasets, confirming the
linear bias term dominates.

\smallskip
\noindent\textbf{Interpretation.}
This is exactly what Theorem~\ref{thm:asymptotic} predicts: the
exploration term $O(b\ell p\sqrt{KT\log(T/\delta)})$ grows as $\sqrt{T\log T}$
and becomes negligible relative to $O(\Delta_\ell T)$ for large~$T$.
The practical implication is that reducing $\Delta_\ell$ (via deeper
codebooks or richer base learners) is more impactful than reducing
exploration cost---consistent with TS-RQ's strong empirical
performance despite its greedy approximation of the joint model.

\begin{figure}[H]
\centering
\includegraphics[width=\textwidth]{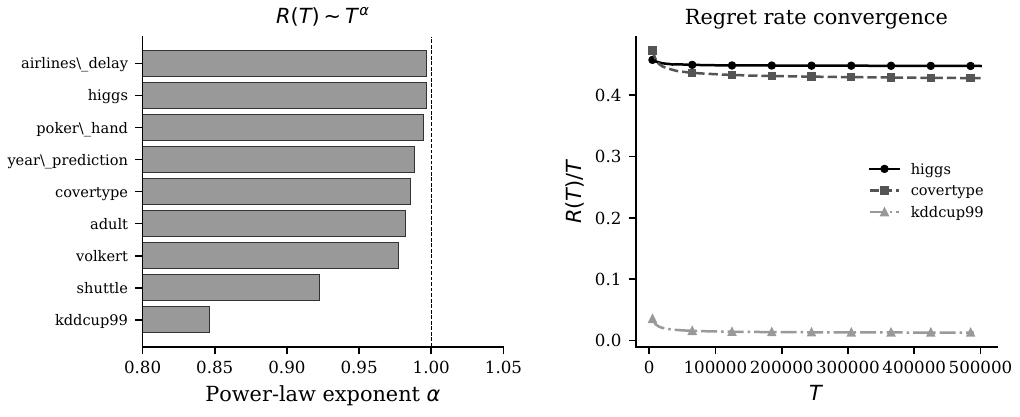}
\caption{Regret scaling analysis for TS-RQ.
\textbf{Left:} Power-law exponent $\alpha$ per dataset
($R(T) \sim T^\alpha$); all values are near $1.0$, indicating
bias-dominated regret.
\textbf{Right:} Regret rate $R(T)/T$ vs.\ $T$ for selected
datasets, converging to a constant consistent with the
$O(\Delta_\ell T)$ bias term in Theorem~\ref{thm:asymptotic}.}
\label{fig:regret_scaling_full}
\end{figure}

\section{Data Sufficiency: Synthetic Validation}\label{app:synthetic}

We validate the data sufficiency condition $N/(bKd^2)$ on synthetic data
with two feature distributions:

\smallskip
\noindent\textbf{Gaussian features.}
$N{=}500\text{K}$ i.i.d.\ standard normal in $\mathbb{R}^d$.
Random per-centroid reward structure.
LinTS-RQ wins when $N/(bKd^2) > 100$, ties at $30{-}100$, loses below $30$.
This is the pessimistic bound: structureless features provide no
exploitable correlation for the codebook.

\smallskip
\noindent\textbf{Structured features.}
Correlated continuous groups, binary categoricals, planted decision tree
for class assignment.
LinTS-RQ wins at \emph{all} tested ratios ($\ge 14$), because the
planted structure aligns with the RQ codebook's partition.

\smallskip
\noindent Real datasets behave like the structured synthetic:
12 of 13 datasets win at ratios as low as $1.1$. The sole failure
(year\_prediction, ratio $0.04$) is $25\times$ below any tested
threshold.

\begin{figure}[H]
\centering
\includegraphics[width=0.8\textwidth]{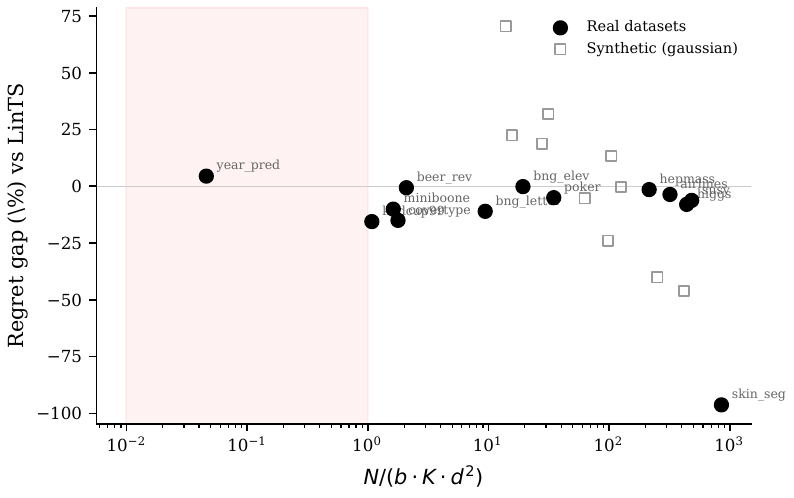}
\caption{Data sufficiency: $N/(bKd^2)$ vs.\ LinTS-RQ regret gap.
Real datasets (filled circles) win at ratios as low as $1.1$.
Synthetic Gaussian (open squares) requires ratio $> 30$.
year\_prediction is the sole outlier at ratio $0.04$.}
\label{fig:sufficiency_appendix}
\end{figure}

\section{year\_prediction Depth Analysis}\label{app:year_depth}

On year\_prediction ($d{=}90$, $K{=}84$), LinTS-RQ degrades even at
depth 1 ($96.3\%$ vs.\ LinTS baseline $93.4\%$). Adding depth makes it
progressively worse ($96.9\%$ at $\ell{=}2$, $97.4\%$ at $\ell{=}4$).

TS-RQ is unaffected: $93.0\%$ at all depths, matching the baseline.
SGD-LinTS-RQ improves slightly with depth ($96.9\% \to 95.0\%$) because
diagonal estimation is more data-efficient than full covariance in $\mathbb{R}^{90}$.

The root cause is data fragmentation: with $b{=}16$ centroids, each
centroid receives $\approx 28\text{K}$ samples to estimate
$d^2 K = 680\text{K}$ parameters---a ratio of $0.04$.
Context-free TS achieves $92.7\%$, outperforming all feature-based
methods, confirming that the features carry negligible signal for the
84-way quantile-binned target.

\section{Reconstruction Error}\label{app:recon}

The normalized reconstruction error
$\|x - \hat{x}\|^2 / \|x\|^2$ at depth 4 ranges from $0.005$
(skin\_segmentation) to $0.648$ (year\_prediction).
year\_prediction's high error confirms that $b{=}16$ centroids in
$\mathbb{R}^{90}$ cannot capture the feature geometry, consistent with
its status as the sole LinTS-RQ failure.

\begin{table}[H]
\centering
\caption{Normalized reconstruction error at depth 4.}
\label{tab:recon}
\begin{tabular}{@{}l r r@{}}
\toprule
\textbf{Dataset} & \textbf{NMSE (d=4)} & \textbf{LinTS-RQ wins?} \\
\midrule
skin\_seg. & 0.005 & Yes ($-96\%$) \\
kddcup99 & 0.018 & Yes ($-16\%$) \\
airlines & 0.078 & Yes ($-3.6\%$) \\
miniboone & 0.111 & Yes ($-10\%$) \\
covertype & 0.185 & Yes ($-15\%$) \\
beer\_rev. & 0.186 & Yes ($-2.9\%$) \\
susy & 0.187 & Yes ($-6.1\%$) \\
poker & 0.196 & Yes ($-5.0\%$) \\
bng\_elev. & 0.211 & Yes ($-0.2\%$) \\
bng\_letter & 0.293 & Yes ($-11\%$) \\
hepmass & 0.400 & Yes ($-1.1\%$) \\
higgs & 0.431 & Yes ($-8.0\%$) \\
year\_pred. & \textbf{0.648} & \textbf{No} ($+4.5\%$) \\
\bottomrule
\end{tabular}
\end{table}

\section{Effective Rank and SGD-LinTS Failure}\label{app:correlation}

On datasets with effective rank $< 1.5$---skin\_segmentation
(eff.\ rank $1.3$, $d{=}3$) and miniboone (eff.\ rank $1.1$,
$d{=}50$)---SGD-LinTS performs \emph{worse than context-free TS}
($42.0\%$ vs $20.8\%$ and $32.9\%$ vs $28.3\%$ regret rate).
The diagonal approximation treats each feature coordinate
independently; when most variance lies on a single axis, it wastes
$d{-}1$ coordinates on noise.
TS-RQ is unaffected: $2.6\%$ on skin\_segmentation and $14.7\%$
on miniboone.

\begin{figure}[H]
\centering
\includegraphics[width=0.7\textwidth]{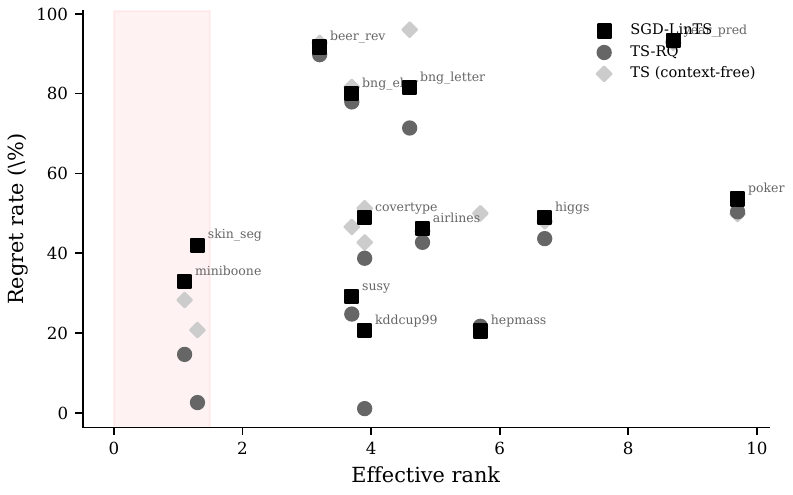}
\caption{Effective rank vs.\ regret rate. SGD-LinTS (squares) fails
severely when eff.\ rank $< 1.5$ (red zone), performing worse
than context-free TS (diamonds). TS-RQ (circles) is unaffected.}
\label{fig:correlation}
\end{figure}

\section{XGBoost and NeuralCB Comparison}\label{app:xgb}

LinTS-RQ achieves regret competitive with both nonlinear replay
baselines. Against Doubling-XGB: 7 wins, 4 losses, 2 ties
(30 LinTS-RQ seeds, 10 XGB seeds, 95\% CIs;
Table~\ref{tab:xgb_detail}). Against Doubling-NeuralCB: 5 wins,
2 losses, 5 ties (30 LinTS-RQ seeds, 10 NeuralCB seeds;
Table~\ref{tab:ncb_detail}). Wins tend to be larger in magnitude
than losses for both comparisons. The key differentiator is memory:
both baselines require an $O(T)$ replay buffer while LinTS-RQ uses
fixed $O(1)$ memory (Table~\ref{tab:memory}).

\begin{table}[H]
\centering
\caption{Per-dataset regret rate (\%): LinTS-RQ vs.\ Doubling-XGB.
Negative gap = LinTS-RQ has lower regret.}
\label{tab:xgb_detail}
\begin{tabular}{@{}l r r r l@{}}
\toprule
\textbf{Dataset} & \textbf{LinTS-RQ} & \textbf{D-XGB} & \textbf{Gap} & \textbf{Call} \\
\midrule
kddcup99          &  0.2 &  9.0 & $-8.8$ & Win \\
bng\_letter       & 46.6 & 56.1 & $-9.5$ & Win \\
covertype         & 23.9 & 28.9 & $-5.0$ & Win \\
beer\_reviews     & 83.1 & 85.7 & $-2.6$ & Win \\
skin\_seg         &  0.3 &  2.9 & $-2.6$ & Win \\
susy              & 20.3 & 21.7 & $-1.4$ & Win \\
hepmass           & 18.4 & 19.7 & $-1.3$ & Win \\
\midrule
bng\_elevators    & 75.0 & 75.1 & $0.0$ & Tie \\
miniboone         &  9.9 &  9.8 & $+0.1$ & Tie \\
\midrule
higgs             & 32.7 & 29.3 & $+3.4$ & Loss \\
airlines\_delay   & 39.8 & 36.0 & $+3.8$ & Loss \\
year\_prediction  & 97.8 & 93.1 & $+4.7$ & Loss \\
poker\_hand       & 46.9 & 41.0 & $+6.0$ & Loss \\
\bottomrule
\end{tabular}
\end{table}

\begin{table}[H]
\centering
\caption{Per-dataset regret rate (\%): LinTS-RQ vs.\ Doubling-NeuralCB
(30 LinTS-RQ seeds, 10 NeuralCB seeds).
Negative gap = LinTS-RQ has lower regret.}
\label{tab:ncb_detail}
\begin{tabular}{@{}l r r r l@{}}
\toprule
\textbf{Dataset} & \textbf{LinTS-RQ} & \textbf{D-NeuralCB} & \textbf{Gap} & \textbf{Call} \\
\midrule
beer\_reviews     & 83.0 & 87.7 & $-4.6$ & Win \\
bng\_letter       & 46.6 & 52.8 & $-6.2$ & Win \\
covertype         & 23.9 & 27.6 & $-3.7$ & Win \\
kddcup99          &  0.2 &  1.4 & $-1.2$ & Win \\
skin\_seg         &  0.3 &  1.6 & $-1.3$ & Win \\
\midrule
airlines\_delay   & 39.8 & 40.2 & $-0.4$ & Tie \\
bng\_elevators    & 75.0 & 75.0 & $+0.1$ & Tie \\
hepmass           & 18.4 & 18.2 & $+0.1$ & Tie \\
higgs             & 32.6 & 32.4 & $+0.3$ & Tie \\
susy              & 20.3 & 20.4 & $-0.1$ & Tie \\
\midrule
poker\_hand       & 46.9 & 45.8 & $+1.1$ & Loss \\
year\_prediction  & 97.8 & 93.1 & $+4.6$ & Loss \\
\bottomrule
\end{tabular}
\end{table}

\begin{figure}[H]
\centering
\includegraphics[width=0.8\textwidth]{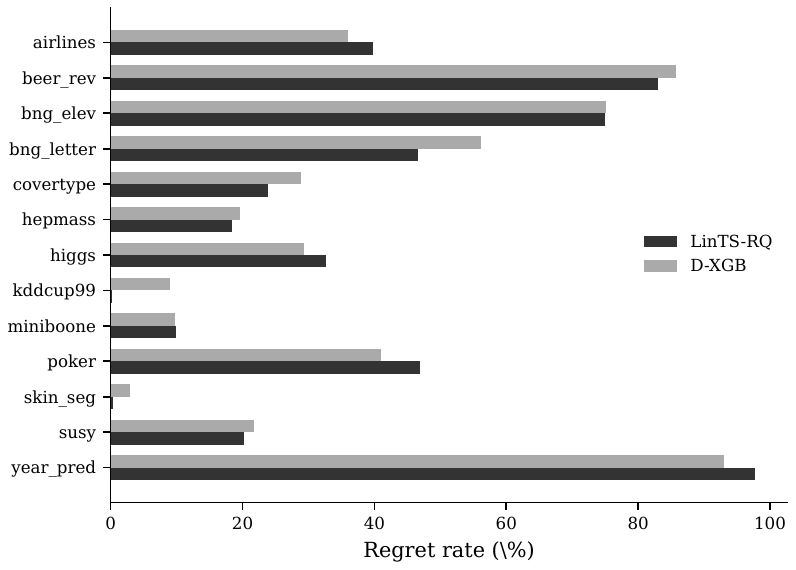}
\caption{Per-dataset regret rate: LinTS-RQ vs.\ Doubling-XGB.}
\label{fig:xgb_comparison}
\end{figure}

\section{TabNet-LinTS Comparison}\label{app:tabnet}

LinTS-RQ wins on 11 of 13 datasets against TabNet-LinTS (11W/2L;
losses: hepmass $+0.5\%$, year\_prediction $+1.7\%$; Welch's
$t$-test, 30 seeds each, 95\% CIs).
Both use the same 50K holdout (unsupervised); the RQ codebook
matches or exceeds TabNet's reconstruction-based embedding.

\begin{figure}[H]
\centering
\includegraphics[width=0.8\textwidth]{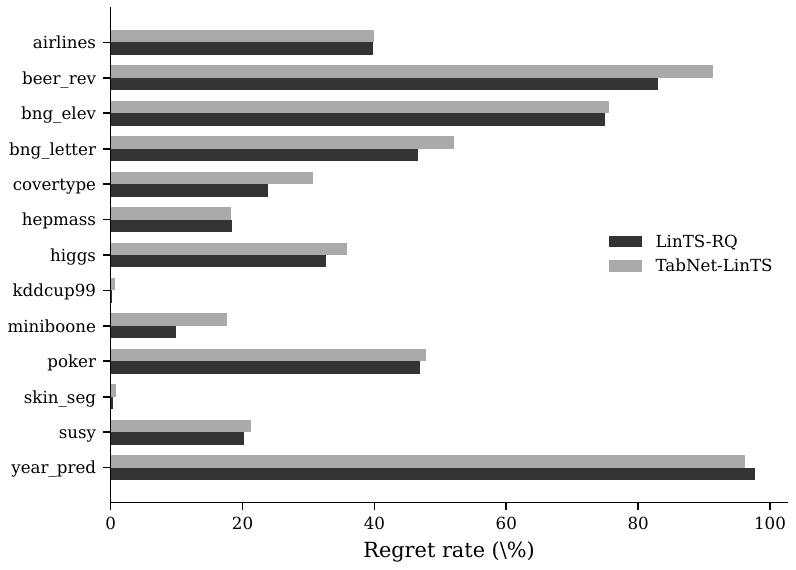}
\caption{Per-dataset regret rate: LinTS-RQ vs.\ TabNet-LinTS.}
\label{fig:tabnet_comparison}
\end{figure}

\section{Timing}\label{app:timing}

\begin{figure}[H]
\centering
\includegraphics[width=\textwidth]{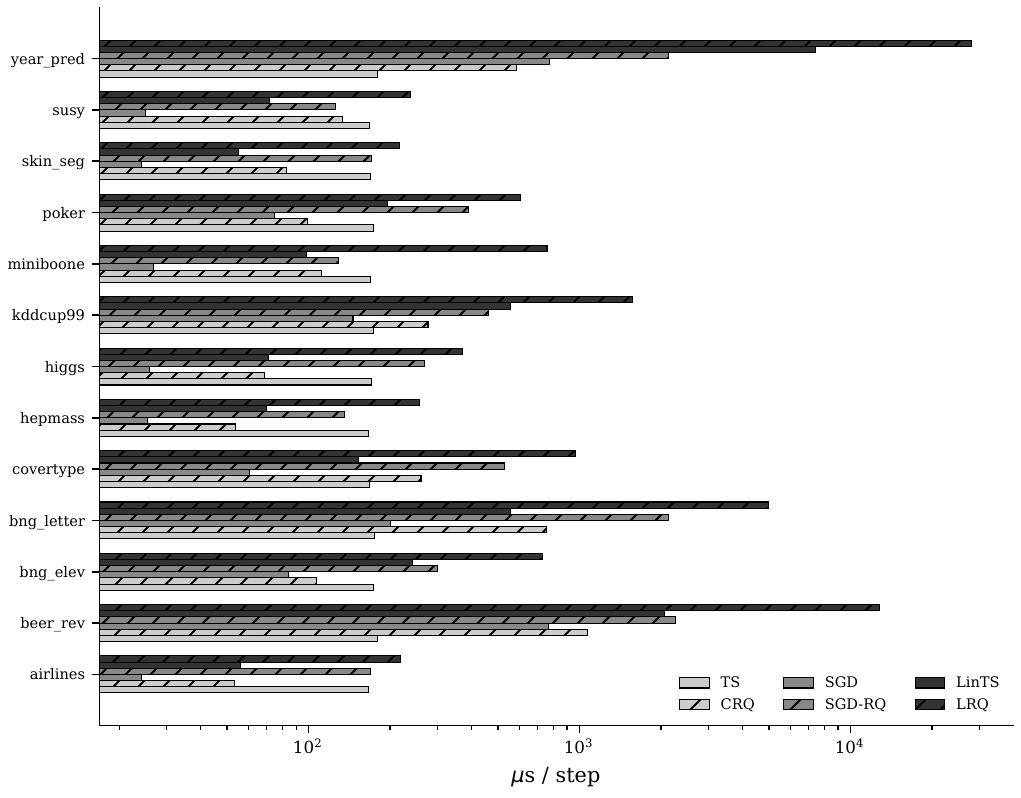}
\caption{Wall-clock $\mu$s/step per method at adapted DRQm depths.
TS-RQ is consistently the cheapest RQ variant ($50{-}750\,\mu$s).
LinTS-RQ on year\_prediction ($d{=}90$, $K{=}84$) reaches $28$\,ms.}
\label{fig:timing}
\end{figure}

\section{Memory Formulas}\label{app:memory}

Model memory = (codebook + model parameters) $\times$ 8 bytes.

\begin{center}
\begin{tabular}{@{}l l l@{}}
\toprule
\textbf{Component} & \textbf{Parameters} & \textbf{Formula} \\
\midrule
Codebook        & $\ell b d$        & Shared across arms \\
TS-RQ      & $2 \ell b K$      & (sum, count) per cell \\
SGD-LinTS-RQ    & $2 \ell b d K$    & (prec, $f$) vectors per cell \\
LinTS-RQ        & $(d^2{+}d) \ell b K$ & ($V$ matrix + $f$ vector) per cell \\
XGB replay      & $T (d{+}2)$       & Grows linearly with $T$ \\
\bottomrule
\end{tabular}
\end{center}

\noindent Adapted depth $\ell$ from shadow promotion varies by dataset
(median across 30 seeds). TS-RQ goes deepest (up to $\ell{=}8$ on
bng\_letter); LinTS-RQ stays shallowest ($\ell{=}2{-}3$, limited by
cascaded convergence).

\begin{tabular}{@{}l r r r r r r r r@{}}
\toprule
\textbf{Dataset} & $d$ & $K$ & TS & TS-RQ & SGD-LinTS & SGD-RQ & LinTS & LinTS-RQ \\
\midrule
airlines\_delay & 7 & 2 & 32B & 4.1KB & 224B & 13.1KB & 896B & 29.8KB \\
beer\_reviews & 12 & 104 & 1.6KB & 82.5KB & 19.5KB & 627.0KB & 126.8KB & 7.9MB \\
bng\_elevators & 18 & 10 & 160B & 9.5KB & 2.8KB & 94.5KB & 26.7KB & 859.5KB \\
bng\_letter & 16 & 26 & 416B & 68.0KB & 6.5KB & 742.0KB & 55.2KB & 5.2MB \\
covertype & 54 & 7 & 112B & 68.0KB & 5.9KB & 405.0KB & 162.4KB & 10.2MB \\
hepmass & 27 & 2 & 32B & 27.1KB & 864B & 50.6KB & 11.8KB & 384.8KB \\
higgs & 28 & 2 & 32B & 12.0KB & 896B & 105.0KB & 12.7KB & 619.5KB \\
kddcup99 & 41 & 17 & 272B & 37.5KB & 10.9KB & 358.8KB & 228.7KB & 7.2MB \\
miniboone & 50 & 2 & 32B & 33.8KB & 1.6KB & 62.5KB & 39.8KB & 2.5MB \\
poker\_hand & 10 & 9 & 144B & 7.0KB & 1.4KB & 95.0KB & 7.7KB & 250.0KB \\
skin\_segmentation & 3 & 2 & 32B & 3.5KB & 96B & 5.6KB & 192B & 6.8KB \\
susy & 18 & 2 & 32B & 16.5KB & 576B & 22.5KB & 5.3KB & 263.2KB \\
year\_prediction & 90 & 84 & 1.3KB & 64.5KB & 118.1KB & 3.7MB & 5.2MB & 252.0MB \\
\bottomrule
\end{tabular}


\newpage
\section*{NeurIPS Paper Checklist}

\begin{enumerate}

\item {\bf Claims}
    \item[] Question: Do the main claims made in the abstract and introduction accurately reflect the paper's contributions and scope?
    \item[] Answer: \answerYes{}
    \item[] Justification: The abstract claims RQ codebooks improve contextual bandits with fixed memory; this is supported by the asymptotic regret analysis (Theorem~\ref{thm:asymptotic}), shadow false-promotion control (Theorem~\ref{thm:shadow}), and experiments across 13 datasets (\S\ref{sec:experiments}).

\item {\bf Limitations}
    \item[] Question: Does the paper discuss the limitations of the work performed by the authors?
    \item[] Answer: \answerYes{}
    \item[] Justification: \S\ref{sec:conclusion} discusses three limitations: the squared-loss approximation for binary rewards, the data sufficiency condition $N/(bKd^2) \gg 1$ (with year\_prediction as a concrete failure case), and the dependence on a holdout set for codebook training.

\item {\bf Theory assumptions and proofs}
    \item[] Question: For each theoretical result, does the paper provide the full set of assumptions and a complete (and correct) proof?
    \item[] Answer: \answerYes{}
    \item[] Justification: Theorem~\ref{thm:asymptotic} (asymptotic regret, Appendix~\ref{app:fixed_depth_proofs}) is conditional on three base-learner convergence properties, which are stated explicitly and verified for all three model classes (Appendix~\ref{app:base_learner_axioms}). Theorem~\ref{thm:shadow} (false-promotion control) is proved in Appendix~\ref{app:promotion_proofs}. Theorem~\ref{thm:depth-rate} exactly decomposes depth-dependent approximation error, and Corollary~\ref{cor:depth-rate} proves geometric decay under explicit reconstruction- and additive-tail contraction assumptions; their proofs are in Appendix~\ref{app:fixed_depth_proofs}.

\item {\bf Experimental result reproducibility}
    \item[] Question: Does the paper fully disclose all the information needed to reproduce the main experimental results of the paper to the extent that it affects the main claims and/or conclusions of the paper (regardless of whether the code and data are provided or not)?
    \item[] Answer: \answerYes{}
    \item[] Justification: Algorithm~\ref{alg:rqbandit} specifies the method. \S\ref{sec:experiments} describes all datasets (publicly available from OpenML), codebook parameters ($b{=}16$, depths 1--5), number of seeds (30), and baselines.

\item {\bf Open access to data and code}
    \item[] Question: Does the paper provide open access to the data and code, with sufficient instructions to faithfully reproduce the main experimental results, as described in supplemental material?
    \item[] Answer: \answerYes{}
    \item[] Justification: Code is available at \url{https://github.com/atavory/rqb}. All 13 datasets are publicly available from OpenML.

\item {\bf Experimental setting/details}
    \item[] Question: Does the paper specify all the training and test details (e.g., data splits, hyperparameters, how they were chosen, type of optimizer) necessary to understand the results?
    \item[] Answer: \answerYes{}
    \item[] Justification: \S\ref{sec:experiments} specifies codebook size ($b{=}16$), clipping parameter ($\sigma$), Thompson sampling noise scale, and the holdout/bandit split. Appendix tables provide per-dataset details.

\item {\bf Experiment statistical significance}
    \item[] Question: Does the paper report error bars suitably and correctly defined or other appropriate information about the statistical significance of the experiments?
    \item[] Answer: \answerYes{}
    \item[] Justification: All experiments use 30 seeds. Figures show $\pm 1$ standard deviation; tables report 95\% CIs. The shadow promotion test uses Freedman's martingale inequality with Bonferroni correction at significance level $\alpha{=}0.05$.

\item {\bf Experiments compute resources}
    \item[] Question: For each experiment, does the paper provide sufficient information on the computer resources (type of compute workers, memory, time of execution) needed to reproduce the experiments?
    \item[] Answer: \answerYes{}
    \item[] Justification: Timing benchmarks are reported in Appendix (Table and Figure). All experiments are CPU-only (no GPU required). Memory usage is analyzed in Appendix.

\item {\bf Code of ethics}
    \item[] Question: Does the research conducted in the paper conform, in every respect, with the NeurIPS Code of Ethics \url{https://neurips.cc/public/EthicsGuidelines}?
    \item[] Answer: \answerYes{}
    \item[] Justification: The research uses only publicly available datasets and proposes a general-purpose bandit algorithm with no foreseeable ethical concerns.

\item {\bf Broader impacts}
    \item[] Question: Does the paper discuss both potential positive societal impacts and negative societal impacts of the work performed?
    \item[] Answer: \answerYes{}
    \item[] Justification: Section~\ref{sec:broader_impact} discusses positive impacts (efficient personalization, resource-constrained deployments) and notes that standard fairness and privacy concerns for bandit algorithms apply.

\item {\bf Safeguards}
    \item[] Question: Does the paper describe safeguards that have been put in place for responsible release of data or models that have a high risk for misuse (e.g., pre-trained language models, image generators, or scraped datasets)?
    \item[] Answer: \answerNA{}
    \item[] Justification: The paper does not release pre-trained models, generative models, or scraped datasets. RQ-Bandit is a lightweight bandit algorithm.

\item {\bf Licenses for existing assets}
    \item[] Question: Are the creators or original owners of assets (e.g., code, data, models), used in the paper, properly credited and are the license and terms of use explicitly mentioned and properly respected?
    \item[] Answer: \answerYes{}
    \item[] Justification: All datasets are cited with their original publications in Appendix~\ref{app:datasets}. All are publicly available from OpenML and UCI under open licenses.

\item {\bf New assets}
    \item[] Question: Are new assets introduced in the paper well documented and is the documentation provided alongside the assets?
    \item[] Answer: \answerYes{}
    \item[] Justification: The RQ-Bandit implementation will be released as supplementary code with documentation and scripts to reproduce the main experiments.

\item {\bf Crowdsourcing and research with human subjects}
    \item[] Question: For crowdsourcing experiments and research with human subjects, does the paper include the full text of instructions given to participants and screenshots, if applicable, as well as details about compensation (if any)?
    \item[] Answer: \answerNA{}
    \item[] Justification: No crowdsourcing or human subjects research was conducted.

\item {\bf Institutional review board (IRB) approvals or equivalent for research with human subjects}
    \item[] Question: Does the paper describe potential risks incurred by study participants, whether such risks were disclosed to the subjects, and whether Institutional Review Board (IRB) approvals (or an equivalent approval/review based on the requirements of your country or institution) were obtained?
    \item[] Answer: \answerNA{}
    \item[] Justification: No human subjects research was conducted.

\item {\bf Declaration of LLM usage}
    \item[] Question: Does the paper describe the usage of LLMs if it is an important, original, or non-standard component of the core methods in this research? Note that if the LLM is used only for writing, editing, or formatting purposes and does \emph{not} impact the core methodology, scientific rigor, or originality of the research, declaration is not required.
    \item[] Answer: \answerNA{}
    \item[] Justification: LLMs were used only for writing assistance and editorial polishing. They are not a component of the RQ-Bandit method or experimental pipeline.

\end{enumerate}

\end{document}